\documentclass{article}
\pdfoutput=1

\PassOptionsToPackage{numbers}{natbib}

\usepackage[preprint]{nips_2018}

\usepackage[utf8]{inputenc} 
\usepackage[T1]{fontenc}    
\usepackage{hyperref}       
\usepackage{url}            
\usepackage{booktabs}       
\usepackage{amsfonts}       
\usepackage{nicefrac}       
\usepackage{microtype}      
\usepackage{graphicx}
\usepackage{subfigure}

\title{Fast greedy algorithms for dictionary selection\\with generalized sparsity constraints}

\author{Kaito Fujii \\ 
Graduate School of Information Sciences and Technology\\ 
The University of Tokyo\\
\texttt{kaito\_fujii@mist.i.u-tokyo.ac.jp}
\And 
Tasuku Soma\\ 
Graduate School of Information Sciences and Technology\\ 
The University of Tokyo\\
\texttt{tasuku\_soma@mist.i.u-tokyo.ac.jp}
}

\usepackage{amsmath,amssymb,amsthm}
\usepackage{cleveref}
\usepackage{algorithm}
\usepackage[noend]{algorithmic}
\renewcommand{\algorithmicrequire}{\textbf{Input:}}
\renewcommand{\algorithmicensure}{\textbf{Output:}}

\newtheorem{theorem}{Theorem}[section]
\newtheorem{lemma}[theorem]{Lemma}
\newtheorem{proposition}[theorem]{Proposition}
\newtheorem{corollary}[theorem]{Corollary}
\theoremstyle{definition}
\newtheorem{example}[theorem]{Example}
\newtheorem{definition}[theorem]{Definition}
\newtheorem{remark}[theorem]{Remark}

\usepackage{mathtools}
\usepackage{enumitem}

\newcommand{\R}{\mathbb{R}}
\newcommand{\Z}{\mathbb{Z}}
\newcommand{\caA}{\mathcal{A}}

\newcommand{\caF}{\mathcal{F}}
\newcommand{\caI}{\mathcal{I}}

\newcommand{\rme}{\mathrm{e}}

\newcommand{\ba}{\mathbf{a}}

\newcommand{\bw}{\mathbf{w}}
\newcommand{\bx}{\mathbf{x}}
\newcommand{\by}{\mathbf{y}}
\newcommand{\bz}{\mathbf{z}}
\newcommand{\bA}{\mathbf{A}}

\newcommand{\bX}{\mathbf{X}}
\newcommand{\bY}{\mathbf{Y}}
\newcommand{\bchi}{\mathbf{e}}

\newcommand{\bfzero}{\mathbf{0}}
\newcommand{\bfone}{\mathbf{1}}
\newcommand{\supp}{\mathrm{supp}}

\DeclarePairedDelimiter{\abs}{\lvert}{\rvert}
\DeclarePairedDelimiter{\norm}{\lVert}{\rVert}
\DeclarePairedDelimiter{\inprod}{\langle}{\rangle}

\DeclareMathOperator{\regret}{regret}
\DeclareMathOperator*{\E}{\mathbf{E}}

\DeclareMathOperator*{\argmax}{argmax}
\newcommand{\sigmamax}{\sigma_{\mathrm{max}}}
\newcommand{\sigmamin}{\sigma_{\mathrm{min}}}
\newcommand{\OPT}{v^*}
\renewcommand{\emptyset}{\varnothing}
\usepackage[table]{xcolor}

\newcommand{\ReplacementOMP}{\textsf{Replacement OMP}}
\newcommand{\ReplacementGreedy}{\textsf{Replacement Greedy}}
\newcommand{\OMPEvaluation}{$\textsf{SDS}_\textsf{OMP}$}
\newcommand{\ModularApproximation}{$\textsf{SDS}_\textsf{MA}$}

\begin{document}

\maketitle

\begin{abstract}
    In dictionary selection, several atoms are selected from finite candidates that successfully approximate given data points in the sparse representation. 
    We propose a novel efficient greedy algorithm for dictionary selection.
    Not only does our algorithm work much faster than the known methods, but it can also handle more complex sparsity constraints, such as average sparsity.
    Using numerical experiments, we show that our algorithm outperforms the known methods for dictionary selection, achieving competitive performances with dictionary learning algorithms in a smaller running time.
\end{abstract}

\allowdisplaybreaks
\section{Introduction}
Learning sparse representations of data and signals has been extensively studied for the past decades in machine learning and signal processing~\citep{Foucart2013}.
In these methods, a specific set of basis signals (atoms), called a \emph{dictionary}, is required and used to approximate a given signal in a sparse representation.
The design of a dictionary is highly nontrivial, and many studies have been devoted to the construction of a good dictionary for each signal domain, such as natural images and sounds.
Recently, approaches to construct a dictionary from data have shown the state-of-the-art results in various domains.
The standard approach is called \emph{dictionary learning}~\citep{Arora2014,Zhou2009,Agarwal2016}.
Although many studies have been devoted to dictionary learning, it is usually difficult to solve, requiring a non-convex optimization problem that often suffers from local minima. 
Also, standard dictionary learning methods (e.g., MOD \citep{Engan1999} or $k$-SVD \citep{Aharon2006}) require a heavy time complexity.

\citet{Krause2010} proposed a combinatorial analogue of dictionary learning, called \emph{dictionary selection}.
In dictionary selection, given a finite set of candidate atoms, a dictionary is constructed by selecting a few atoms from the set. 
Dictionary selection could be faster than dictionary learning due to its discrete nature.
Another advantage of dictionary selection is that the approximation guarantees hold even in agnostic settings, i.e., we do not need stochastic generating models of the data.
Furthermore, dictionary selection algorithms can be used for \emph{media summarization}, in which the atoms must be selected from given data points~\cite{Cong2012,Cong2017}.

The basic dictionary selection is formalized as follows. 
Let $V$ be a finite set of candidate atoms and $n = \abs{V}$. 
Throughout the paper, we assume that the atoms are unit vectors in $\R^d$ without loss of generality.
We represent the candidate atoms as a matrix $\bA \in \R^{d\times n}$ whose columns are the atoms in $V$.
Let $\by_t \in \R^d$ ($t \in [T]$) be data points, where $[T] = \{1, \dots, T\}$, and $k$ and $s$ be positive integers with $k \geq s$.
We assume that a utility function $u : \R^d  \times \R^d \to \R_+$ exists, which measures the similarity of the input vectors.
For example, one can use the $\ell^2$-utility function $u(\by, \bx) = \norm{\by}_2^2 - \norm{\by - \bx}_2^2$ as in \citet{Krause2010}.
Then, the dictionary selection finds a set $X \subseteq V$ of size $k$ that maximizes 
\begin{align}
    h(X) = \sum_{t = 1}^T \max_{\bw \in \R^k \colon \norm{\bw}_0 \leq s} u(\by_t, \bA_X \bw),
\end{align}
where $\norm{\bw}_0$ is the number of nonzero entries in $\bw$ and $\bA_X$ is the column submatrix of $\bA$ with respect to $X$.
That is, we approximate a data point $\by_t$ with a sparse representation in atoms in $X$, where the approximation quality is measured by $u$.
Letting $f_t(Z_t) := \max_{\bw} u(\by_t, \bA_{Z_t}\bw)$ ($t \in [T]$), we can rewrite this as the following \emph{two-stage optimization}:
$h(X) = \sum_{t = 1}^T \max_{Z_t \subseteq X \colon \abs{Z_t} \leq s} f_t(Z_t)$.
Here $Z_t$ is the set of atoms used in a sparse representation of data point $\by_t$. 
The main challenges in dictionary selection are that the evaluation of $h$ is NP-hard in general~\citep{Natarajan1995}, and the objective function $h$ is not submodular~\citep{Fujishige2005} and therefore the well-known greedy algorithm~\citep{Nemhauser1978a} cannot be applied. 
The previous approaches construct a good proxy of dictionary selection that can be easily solved, and analyze the approximation ratio.





\subsection{Our contribution}
Our main contribution is a novel and efficient algorithm called the \emph{replacement orthogonal matching pursuit (\ReplacementOMP{})}  for dictionary selection.
This algorithm is based on a previous approach called \ReplacementGreedy{}~\citep{Stan2017} for \emph{two-stage submodular maximization}, a similar problem to dictionary selection.
However, the algorithm was not analyzed for dictionary selection.
We extend their approach to dictionary selection in the present work, with an additional improvement that exploits techniques in orthogonal matching pursuit.
We compare our method with the previous methods in \Cref{table:comparison}.
\ReplacementOMP{} has a smaller running time than \OMPEvaluation~\citep{Das2011} and \ReplacementGreedy. 
The only exception is \ModularApproximation{}~\citep{Das2011}, which intuitively ignores any correlation of the atoms. 
In our experiment, we demonstrate that \ReplacementOMP{} outperforms \ModularApproximation{} in terms of test residual variance.
We note that the constant approximation ratios of \ModularApproximation{}, \ReplacementGreedy, and \ReplacementOMP{} are incomparable in general.
In addition, we demonstrate that \ReplacementOMP{} achieves a competitive performance with dictionary learning algorithms in a smaller running time, in numerical experiments.

\paragraph{Generalized sparsity constraint}
Incorporating further prior knowledge on the data domain often improves the quality of dictionaries~\citep{Rubinstein2010,Rusu2014,Dumitrescu2018dictionary}.
A typical example is a combinatorial constraint independently imposed on each support $Z_t$. This can be regarded as a natural extension of the \emph{structured sparsity}~\cite{Huang2009} in sparse regression, which requires the support to satisfy some combinatorial constraint, rather than a cardinality constraint.
A \emph{global structure} of supports is also useful prior information.
\citet{Cevher2011} proposed a global sparsity constraint called the \emph{average sparsity}, in which they add a global constraint $\sum_{t=1}^T \abs{Z_t} \leq s'$.
Intuitively, the average sparsity constraint requires that the most data points can be represented by a small number of atoms. 
If the data points are patches of a natural image, most patches are a simple background, and therefore the number of the total size of the supports must be small.
The average sparsity has been also intensively studied in dictionary learning~\cite{Dumitrescu2018dictionary}.
To deal with these generalized sparsities in a unified manner, we propose a novel class of sparsity constraints, namely \textit{$p$-replacement sparsity families}. 
We prove that \ReplacementOMP{} can be applied for the generalized sparsity constraint with a slightly worse approximation ratio.
We emphasize that the OMP approach is essential for \emph{efficiency};
in contrast, \ReplacementGreedy{} cannot be extended to the average sparsity setting because it can only handle local constraints on $Z_t$, and yields an exponential running time.

{
\begin{table}[t]
\rowcolors{1}{gray!15}{white}
    \centering
    \caption{Comparison of known methods with \ReplacementOMP. The constants $m_s$, $M_s$, and $M_{s,2}$ are the restricted concavity and smoothness constants of $u(\by_t, \cdot)$ ($t \in [T]$); see \Cref{sec:pre}. The running time is from the $\ell^2$-utility function $u$ and the individual sparsity constraint.}\label{table:comparison}
\begin{tabular}{cccp{1.5cm}}
        \hline
        Method & Approximation ratio & Running time & Generalized sparsity \\\hline 
        \ModularApproximation~\cite{Krause2010} & $\frac{m_1 m_s}{M_1 M_s}  (1-1/\mathrm{e})$~\cite{Das2011} & $\mathrm{O}((k+d)nT)$ & No \\
        \OMPEvaluation~\cite{Krause2010} & $\mathrm{O}(1/k)$~\cite{Das2011} & $\mathrm{O}((s+k)sdknT)$ & No \\
        \ReplacementGreedy~\cite{Stan2017} & $\left(\frac{m_{2s}}{M_{s,2}}\right)^2\left(1-\exp\left(-\frac{M_{s,2}}{m_{2s}}\right)\right)$ & $\mathrm{O}(s^2dknT)$ & No \\
        \bf this paper & $\left(\frac{m_{2s}}{M_{s,2}}\right)^2\left(1-\exp\left(-\frac{M_{s,2}}{m_{2s}}\right)\right)$ & \bf $\mathrm{O}((n+ds)kT)$ & Yes
    \end{tabular}
\end{table}
}

\paragraph{Online extension}
In some practical situations, it is not always feasible to store all data points $\by_t$, but these data points arrive in an online fashion.
We show that \ReplacementOMP{} can be extended to the online setting, with a sublinear approximate regret. 
The details are given in \Cref{sec:extension}.

\subsection{Related work}
\citet{Krause2010} first introduced dictionary selection as a combinatorial analogue of dictionary learning. 
They proposed \ModularApproximation{} and \OMPEvaluation{}, and analyzed the approximation ratio using the \emph{coherence} of the matrix $\bA$.
\citet{Das2011} introduced the concept of the \emph{submodularity ratio} and refined the analysis via the \emph{restricted isometry property}~\citep{Candes2005}. 
A connection to the restricted concavity and submodularity ratio has been investigated by \citet{Elenberg2016,Khanna2017} for sparse regression and matrix completion.
\citet{Balkanski2016} studied two-stage submodular maximization as a submodular proxy of dictionary selection, devising various algorithms.
\citet{Stan2017} proposed \ReplacementGreedy{} for two-stage submodular maximization.
It is unclear that these methods provide an approximation guarantee for the original dictionary selection.

To the best of our knowledge, there is no existing research in the literature that addresses online dictionary selection.
For a related problem in sparse optimization, namely \emph{online linear regression}, \citet{Kale2017} proposed an algorithm based on \emph{supermodular minimization}~\citep{Liberty2017} with a sublinear approximate regret guarantee.
\citet{Elenberg2017} devised a streaming algorithm for weak submodular function maximization. 
\citet{Chen2018} dealt with online maximization of weakly DR-submodular functions.

\paragraph{Organization}
The rest of this paper is organized as follows.
\Cref{sec:pre} provides the basic concepts and definitions.
\Cref{sec:general} formally defines dictionary selection with generalized sparsity constraints.
\Cref{sec:alg} presents our algorithm, \ReplacementOMP{}. 
\Cref{sec:extension} sketches the extension to the online setting.
The experimental results are presented in \Cref{sec:exp}.

\section{Preliminaries}\label{sec:pre}

\paragraph{Notation}
For a positive integer $n$, $[n]$ denotes the set $\{1,2,\dots, n\}$.
The sets of reals and nonnegative reals are denoted by $\R$ and $\R_{\ge 0}$, respectively.
We similarly define $\Z$ and $\Z_{\ge 0}$.
Vectors and matrices are denoted by lower and upper case letters in boldface, respectively: $\ba, \bx, \by$ for vectors and $\bA, \bX, \bY$ for matrices.
The $i$th standard unit vector is denoted by $\bchi_i$; 
that is, $\bchi_i$ is the vector such that its $i$th entry is equal to one and all other entries are zero.
For a matrix $\bA \in \R^{d \times n}$ and $X \subseteq [n]$, $\bA_X$ denotes the column submatrix of $\bA$ with respect to $X$.
The maximum and minimum singular values of a matrix $\bA$ are denoted by $\sigmamax(\bA)$ and $\sigmamin(\bA)$, respectively.
For a positive integer $k$, we define $\sigmamax(\bA, k) \coloneqq \max_{X \subseteq [n]\colon \abs{X} \leq k} \sigmamax(\bA_X)$. 
We define $\sigmamin(\bA, k)$ in a similar way.
For $t \in [T]$, let $u_t(\bw):= u(\by_t, \bA \bw)$.
Let $\bw^{(Z_t)}_t$ denote the maximizer of $u_t(\bw)$ subject to $\supp(\bw) \subseteq Z_t$. 
Throughout the paper, $V$ denotes the fixed finite ground set.
For $X \subseteq V$ and $a \in V \setminus X$, we define $X + a \coloneqq X \cup \{a\}$.
Similarly, for $a \in V \setminus X$ and $b \in X$, we define $X - b + a \coloneqq (X \setminus \{b\})\cup\{a\}$.

\subsection{Restricted concavity and smoothness}
The following concept of restricted strong concavity and smoothness is crucial in our analysis.
    \begin{definition}[{Restricted strong concavity and restricted smoothness~\citep{negahban2012}}]
	Let $\Omega$ be a subset of $\R^d \times \R^d$ and $u \colon \R^d \to \R$ be a continuously differentiable function.
	We say that $u$ is \emph{restricted strongly concave} with parameter $m_\Omega$ and \emph{restricted smooth} with parameter $M_\Omega$ if,
	\begin{align*}
			- \frac{m_\Omega}{2} \norm{\by - \bx}_2^2
			\ge u(\by) - u(\bx) - \inprod{\nabla u (\bx), \by - \bx} 
			\ge - \frac{M_\Omega}{2} \norm{\by - \bx}^2_2
	\end{align*}
    for all $(\bx, \by) \in \Omega$.
\end{definition}

We define $\Omega_{s, p} \coloneqq \{ (\bx, \by) \in \R^d \times \R^d \colon \norm{\bx}_0, \norm{\by}_0 \leq s, \norm{\bx - \by}_0 \leq p \}$ and $\Omega_s \coloneqq \Omega_{s,s}$ for positive integers $s$ and $p$.
We often abbreviate $M_{\Omega_s}$, $M_{\Omega_{s,p}}$, and $m_{\Omega_s}$ as $M_{s}$, $M_{s,p}$, and $m_{s}$, respectively.

\section{Dictionary selection with generalized sparsity constraints}\label{sec:general}
In this section, we formalize our problem, \emph{dictionary selection with generalized sparsity constraints}. 
In this setting, the supports $Z_t$ for each $t \in [T]$ cannot be independently selected, but we impose a global constraint on them.
We formally write such constraints as a down-closed~\footnote{A set family $\caI$ is said to be down-closed if $X \in \caI$ and $Y \subseteq X$ then $Y \in \caI$.} family $\mathcal{I} \subseteq \prod_{t = 1}^T 2^V$.
Therefore, we aim to find $X \subseteq V$ with $\abs{X} \leq k$ maximizing 
\begin{align}
    h(X) = \max_{Z_1, \dots, Z_t \subseteq X \mid (Z_1, \dots, Z_t) \in \caI} \sum_{t=1}^T f_t(Z_t) 
\end{align}

Since a general down-closed family is too abstract, we focus on the following class.
First, we define the set of \emph{feasible replacements} for the current support $Z_1,\cdots,Z_T$ and an atom $a$ as
\begin{equation}\label{eq:feasible-replacement}
    \mathcal{F}_{a}(Z_1,\cdots,Z_T)  = \left\{(Z'_1, \cdots, Z'_T) \in \mathcal{I} \mid Z'_t \subseteq Z_T + a, \, |Z_t \setminus Z'_t| \le 1 ~ (\forall t \in [T]) \right\}.
\end{equation}
That is, the set of members in $\caI$ obtained by adding $a$ and removing at most one element from each $Z_t$.
Let $\mathcal{F}(Z_1,\cdots,Z_T) = \bigcup_{a \in V} \mathcal{F}_{a}(Z_1,\cdots,Z_T)$.
If $Z_1, \dots, Z_T$ are clear from the context, we simply write it as $\caF_{a}$.

\begin{definition}[{$p$-replacement sparsity}]
    A sparsity constraint $\mathcal{I} \subseteq \prod_{t = 1}^T 2^V$ is \emph{$p$-replacement sparse} if for any $(Z_1,\dots,Z_T), (Z_1^*, \dots, Z_T^*) \in \mathcal{I}$, there is a sequence of $p$ feasible replacements $(Z^{p'}_1,\dots,Z^{p'}_T) \in \caF(Z_1,\dots,Z_T)$ ($p' \in [p]$) such that each element in $Z^*_t \setminus Z_t$ appears at least once in the sequence $(Z^{p'}_t \setminus Z_t)_{p'=1}^p$ and each element in $Z_t \setminus Z_t^*$ appears at most once in the sequence $(Z_t \setminus Z^{p'}_t)_{p'=1}^{p}$.
\end{definition}

The following sparsity constraints are all $p$-replacement sparsity families.
See \Cref{sec:general-appendix} for proof.

\begin{example}[individual sparsity]
The sparsity constraint for the standard dictionary selection can be written as $\mathcal{I} = \{(Z_1,\cdots,Z_T) \mid |Z_t| \le s ~ (\forall t \in [T])\}$. 
We call it \textit{the individual sparsity constraint}. 
This constraint is a special case of an individual matroid constraint, described below.
\end{example}

\begin{example}[individual matroids]
This was proposed by \citep{Stan2017} as a sparsity constraint for two-stage submodular maximization. 
An \emph{individual matroid constraint} can be written as $\mathcal{I} = \{(Z_1,\cdots,Z_T) \mid Z_t \in \mathcal{I}_t ~ (\forall t \in [T])\}$ where $(V, \mathcal{I}_t)$ is a matroid\footnote{A \textit{matroid} is a pair of a finite ground set $V$ and a non-empty down-closed family $\mathcal{I} \subseteq 2^V$ that satisfy that for all $Z, Z' \in \mathcal{I}$ with $|Z| < |Z'|$, there is an element $a \in Z' \setminus Z$ such that $Z \cup \{a\} \in \mathcal{I}$} for each $t \in [T]$. 
An individual sparsity constraint is a special case of an individual matroid constraint where $(V, \mathcal{I}_t)$ is the uniform matroid for all $t$.
\end{example}

\begin{example}[block sparsity]
Block sparsity was proposed by \citet{Krause2010}. 
This sparsity requires that the support must be sparse within each prespecified block. 
That is, disjoint blocks $B_1,\cdots,B_b \subseteq [T]$ of data points are given in advance, and an only small subset of atoms can be used in each block. 
    Formally, $\mathcal{I} = \{(Z_1,\cdots,Z_T) \mid | \bigcup_{t \in B_{b'}} Z_t | \le s_{b'}~(\forall b' \in [b]) \}$ where $s_{b'} \in \Z_+$ for each $b' \in [b]$ are sparsity parameters.
\end{example}

\begin{example}[{average sparsity~\citep{Cevher2011}}] 
This sparsity imposes a constraint on the average number of used atoms among all data points. The number of atoms used for each data point is also restricted. 
Formally, $\mathcal{I} = \{(Z_1,\cdots,Z_T) \mid |Z_t| \le s_t, \sum_{t = 1}^T |Z_t| \le s' \}$ where $s_t \in \Z_{+}$ for each $t \in [T]$ and $s' \in \Z_{+}$ are sparsity parameters.
\end{example}

\begin{proposition}
    The replacement sparsity parameters of individual matroids, block sparsity, and average sparsity are upper-bounded by $k$, $k$, and $3k - 1$, respectively. 
\end{proposition}

\section{Algortihms}\label{sec:alg}
In this section, we present \ReplacementGreedy~\citep{Stan2017} and \ReplacementOMP{} for dictionary selection with generalized sparsity constraints. 

\subsection{Replacement Greedy}
\ReplacementGreedy{} was first proposed as an algorithm for a different problem, \emph{two-stage submodular maximization}~\citep{Balkanski2016}.
In two-stage submodular maximization, the goal is to maximize 
\begin{align}
    h(X) = \sum_{t = 1}^T \max_{Z_t \subseteq X \colon Z_t \in \caI_t} f_t(Z_t),
\end{align}
where $f_t$ is a nonnegative monotone submodular function ($t \in [T]$) and $\caI_t$ is a matroid.
Despite the similarity of the formulation, in dictionary selection, the functions $f_t$ are not necessarily submodular, but come from the continuous function $u_t$.
Furthermore, in two-stage submodular maximization, the constraints on $Z_t$ are individual for each $t \in [T]$, while we pose a global constraint $\caI$. 
In the following, we present an adaptation of \ReplacementGreedy{} to dictionary selection with generalized sparsity constraints.

\ReplacementGreedy{} stores the current dictionary $X$ and supports $Z_t \subseteq X$ such that $(Z_1, \dots, Z_T) \in \caI$, which are initialized as $X = \emptyset$ and $Z_t = \emptyset$ ($t\in [T]$). 
At each step, the algorithm considers the gain of adding an element $a \in V$ to $X$ with respect to each function $f_t$,
i.e., the algorithm selects $a$ that maximizes $\max_{(Z_1', \dots, Z_T') \in \caF_a} \sum_{t=1}^T \{f_t(Z_t') - f(Z_t) \}$.
See \Cref{alg:replacement-greedy} for a pseudocode description.
Note that for the individual matroid constraint $\caI$, the algorithm coincides with the original \ReplacementGreedy~\citep{Stan2017}.

\begin{algorithm}
    \caption{\ReplacementGreedy{} \& \ReplacementOMP}\label{alg:replacement-greedy}
\begin{algorithmic}[1]
	\STATE Initialize $X \gets \emptyset$ and $Z_t \gets \emptyset$ for $t = 1, \dots, T$.
    \FOR{$i = 1, \dots, k$}
        \STATE Pick $a^* \in V$ that maximizes\\
            $
            \begin{cases}
                \max_{(Z'_1,\cdots,Z'_T) \in \mathcal{F}_{a^*}} \sum_{t = 1}^T \left\{ f_t(Z'_t) - f_t(Z_t) \right\} \quad \text{(\ReplacementGreedy{})}\\
                \max_{(Z'_1,\cdots,Z'_T) \in \mathcal{F}_{a^*}} \left\{ \frac{1}{M_{s,2}} \sum_{t = 1}^T \norm{\nabla u_t(\bw^{(Z_t)}_t)_{Z'_t \setminus Z_t}}^2 - M_{s,2} \sum_{t = 1}^T \norm{(\bw^{(Z_t)}_t)_{Z_t \setminus Z'_t}}^2 \right\}\\
            \end{cases}
            $
                \\\hfill(\ReplacementOMP{})\\
        and let $(Z'_1, \cdots, Z'_T)$ be a replacement achieving a maximum.
        \STATE Set $X \gets X + a^*$ and $Z_t \gets Z'_t$ for each $t \in [T]$.
    \ENDFOR
    \STATE \textbf{return} $X$.
\end{algorithmic}
\end{algorithm}

\citet{Stan2017} showed that \ReplacementGreedy{} achieves an $((1-1/\sqrt{e})/2)$-approximation when $f_t$ are monotone submodular.
We extend their analysis to our non-submodular setting.
The proof can be found in \Cref{sec:alg-appendix}.

\begin{theorem}\label{thm:replacement-greedy-1}
    Assume that $u_t$ is $m_{2s}$-strongly concave on $\Omega_{2s}$ and $M_{s,2}$-smooth on $\Omega_{s,2}$ for $t \in [T]$ and that the sparsity constraint $\mathcal{I}$ is $p$-replacement sparse.
    Let $(Z_1^*,\cdots,Z_T^*) \in \caI$ be optimal supports of an optimal dictionary $X^*$.
    Then the solution $(Z_1, \cdots, Z_T) \in \caI$ of \ReplacementGreedy{} after $k'$ steps satisfies
    \begin{equation*}
        \sum_{t = 1}^T f_t(Z_t) \ge \frac{m^2_{2s}}{M_{s,2}^2} \left( 1 - \exp \left( - \frac{k'}{p} \frac{M_{s,2}}{m_{2s}} \right) \right) \sum_{t = 1}^T f_t(Z^*_t)
    \end{equation*}
\end{theorem}


\subsection{Replacement OMP}\label{sec:replacement-OMP}
Now we propose our algorithm, \ReplacementOMP{}.
A down-side of \ReplacementGreedy{} is its heavy computation:
in each greedy step, we need to evaluate $\sum_{t=1}^Tf_t(Z_t')$ for each $(Z_1',\dots,Z_t') \in\caF_a(Z_1,\dots,Z_t)$, which amounts to solving linear regression problems $snT$ times if $u$ is the $\ell^2$-utility function.
To avoid heavy computation, we propose a proxy of this quantity by borrowing an idea from orthogonal matching pursuit. 
\ReplacementOMP{} selects an atom $a \in V$ that maximizes
\begin{equation}\label{eq:replacement-omp-marginal-gain-1}
    \max_{(Z'_1,\cdots,Z'_T) \in \mathcal{F}_{a}(Z_1,\cdots,Z_T)} \left\{ \frac{1}{M_{s,2}} \sum_{t = 1}^T \norm{\nabla u_t(\bw^{(Z_t)}_t)_{Z'_t \setminus Z_t}}^2 - M_{s,2} \sum_{t = 1}^T \norm{(\bw^{(Z_t)}_t)_{Z_t \setminus Z'_t}}^2 \right\}.
\end{equation}
This algorithm requires the smoothness parameter $M_{s,2}$ before the execution.
Computing $M_{s,2}$ is generally difficult, but this parameter for the squared $\ell^2$-utility function can be bounded by $\sigmamax^2(\bA, 2)$. 
This value can be computed in $O(n^2d)$ time.

\begin{theorem}\label{thm:replacement-omp-general}
    Assume that $u_t$ is $m_{2s}$-strongly concave on $\Omega_{2s}$ and $M_{s,2}$-smooth on $\Omega_{s,2}$ for $t \in [T]$ and that the sparsity constraint $\mathcal{I}$ is $p$-replacement sparse.
    Let $(Z_1^*,\cdots,Z_T^*) \in \caI$ be optimal supports of an optimal dictionary $X^*$.
    Then the solution $(Z_1, \cdots, Z_T) \in \caI$ of \ReplacementOMP{} after $k'$ steps satisfies
    \begin{equation*}
        \sum_{t = 1}^T f_t(Z_t) \ge \frac{m^2_{2s}}{M_{s,2}^2} \left( 1 - \exp \left( - \frac{k'}{p} \frac{M_{s,2}}{m_{2s}} \right) \right) \sum_{t = 1}^T f_t(Z^*_t).
    \end{equation*}
\end{theorem}

\subsection{Complexity}
Now we analyze the time complexity of \ReplacementGreedy{} and \ReplacementOMP.
In general, $\caF_a$ has $O(n^T)$ members, and therefore it is difficult to compute $\caF_a$.
Nevertheless, we show that \ReplacementOMP{} can run much faster for the examples presented in \Cref{sec:general}. 

In \ReplacementGreedy{}, it is difficult to find an atom with the largest gain at each step.
This is because we need to maximize a nonlinear function $\sum_{t=1}^T f_t(Z_t')$.
Conversely, in \ReplacementOMP{}, if we can calculate $\bw^{(Z_t)}_t$ and $\nabla u_t(\bw^{(Z_t)}_t)$ for all $t \in [T]$, the problem of calculating gain of each atom is reduced to maximizing a linear function.

In the following, we consider the $\ell_2$-utility function and average sparsity constraint because it is the most complex constraint.
A similar result holds for the other examples.
In fact, we show that this task reduces to maximum weighted bipartite matching.
The Hungarian method returns the maximum weight bipartite matching in $\mathrm{O}(T^3)$ time.
We can further improve the running time to $O(T\log T)$ time by utilizing the structure of this problem. 
Due to the limitation of space, we defer the details to \Cref{sec:alg-appendix}.
In summary, we obtain the following:
\begin{theorem}\label{thm:replacement-omp-runtime}
    Assume that the assumption of \Cref{thm:replacement-omp-general} holds. 
    Further assume that $u$ is the $\ell^2$-utility function and $\caI$ is the average sparsity constraint.
    Then \ReplacementOMP{} finds the solution $(Z_1, \cdots, Z_T) \in \caI$ 
    \begin{equation*}
        \sum_{t = 1}^T f_t(Z_t) \ge \left( \frac{\sigmamax^2(\bA, 2s)}{\sigmamin^2(\bA,2)} \right)^2 \left( 1 - \exp \left( - \frac{1}{3} \frac{\sigmamin^2(\bA,2)}{\sigmamax^2(\bA, 2s)}  \right) \right) \sum_{t = 1}^T f_t(Z^*_t)
    \end{equation*}
    in $O(Tk(n \log T + ds))$ time. 
\end{theorem}

\begin{remark}
    If finding an atom with the largest gain is computationally intractable, we can add an atom whose gain is no less than $\tau$ times the largest gain. In this case, we can bound the approximation ratio with replacing $k'$ with $\tau k'$ in \Cref{thm:replacement-greedy-1} and \ref{thm:replacement-omp-general}.
\end{remark}

\section{Extensions to the online setting}\label{sec:extension}
Our algorithms can be extended to the following online setting.
The problem is formalized as a two-player game between a player and an adversary.
At each round $t = 1, \dots, T$, the player must select (possibly in a randomized manner) a dictionary $X_t \subseteq V$ with $\abs{X_t} \leq k$.
Then, the adversary reveals a data point $\by_t \in \R^d$ and the player gains 
$f_t(X_t) = \max_{\bw \in \R^k: \norm{\bw}_0 \leq s}u(\by_t, \bA_X\bw)$. 
The performance measure of a player's strategy is the \emph{expected $\alpha$-regret}:
\begin{align*}
    \regret_\alpha(T) = \alpha\max_{X^*: \abs{X^*} \leq k}\sum_{t=1}^T f_t(X^*) - \E\left[ \sum_{t=1}^T f_t(X_t)\right],
\end{align*}
where $\alpha > 0$ is a constant independent from $T$ corresponding to the offline approximation ratio, and the expectation is taken over the randomness in the player.

For this online setting, we present an extension of \ReplacementGreedy{} and \ReplacementOMP{} with sublinear $\alpha$-regret, where $\alpha$ is the corresponding offline approximation ratio.
The details are provided in \Cref{sec:online}.

\section{Experiments}\label{sec:exp}

\begin{figure*}
\centering
\subfigure[synthetic, $T = 100$, time]{
	\includegraphics[width=0.31\textwidth]{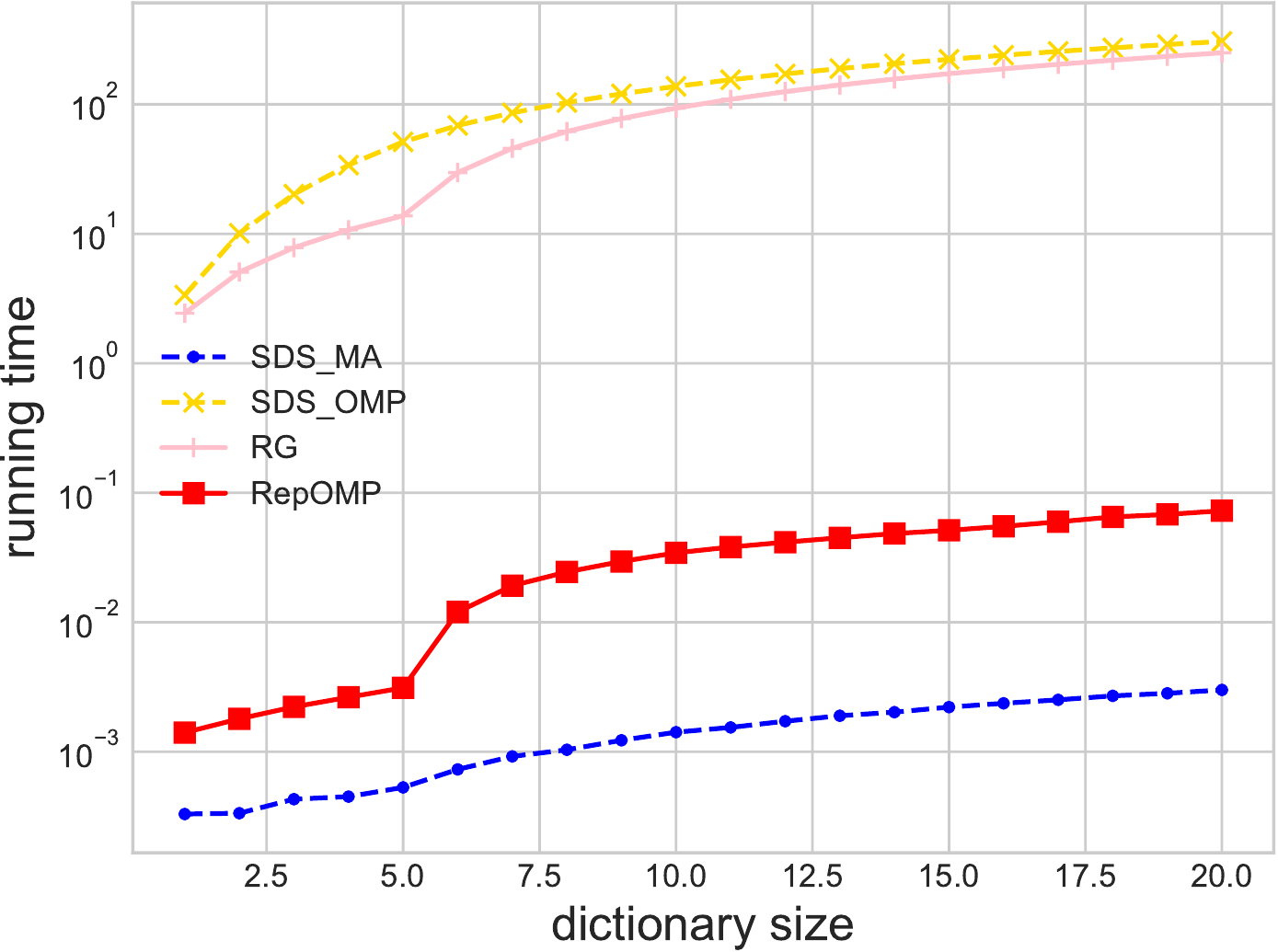}\label{fig:offline_synthetic_T100_time}
}
\subfigure[synthetic, $T = 100$, residual]{
	\includegraphics[width=0.31\textwidth]{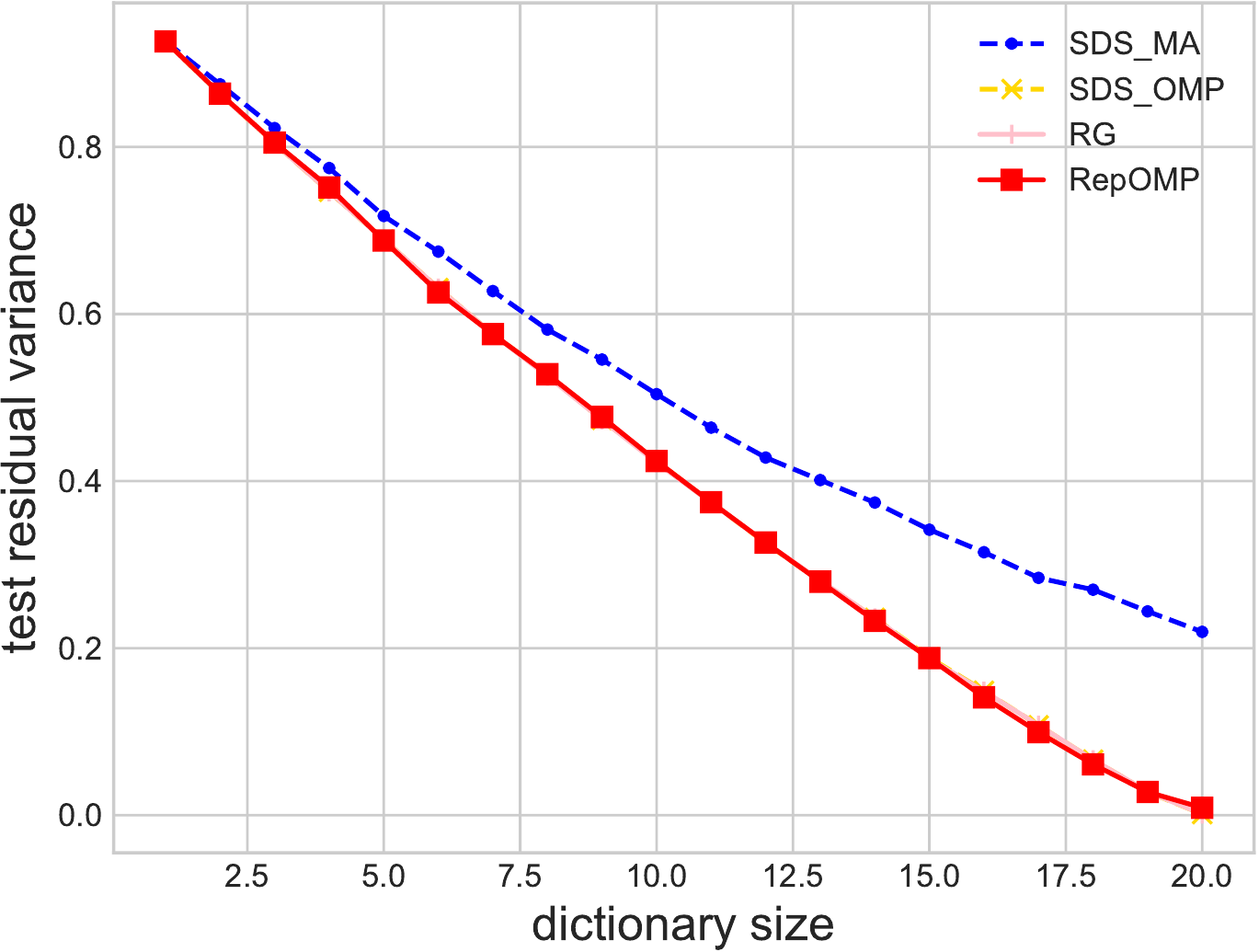}\label{fig:offline_synthetic_T100_error}
}
\subfigure[voc, $T = 100$, residual]{
	\includegraphics[width=0.31\textwidth]{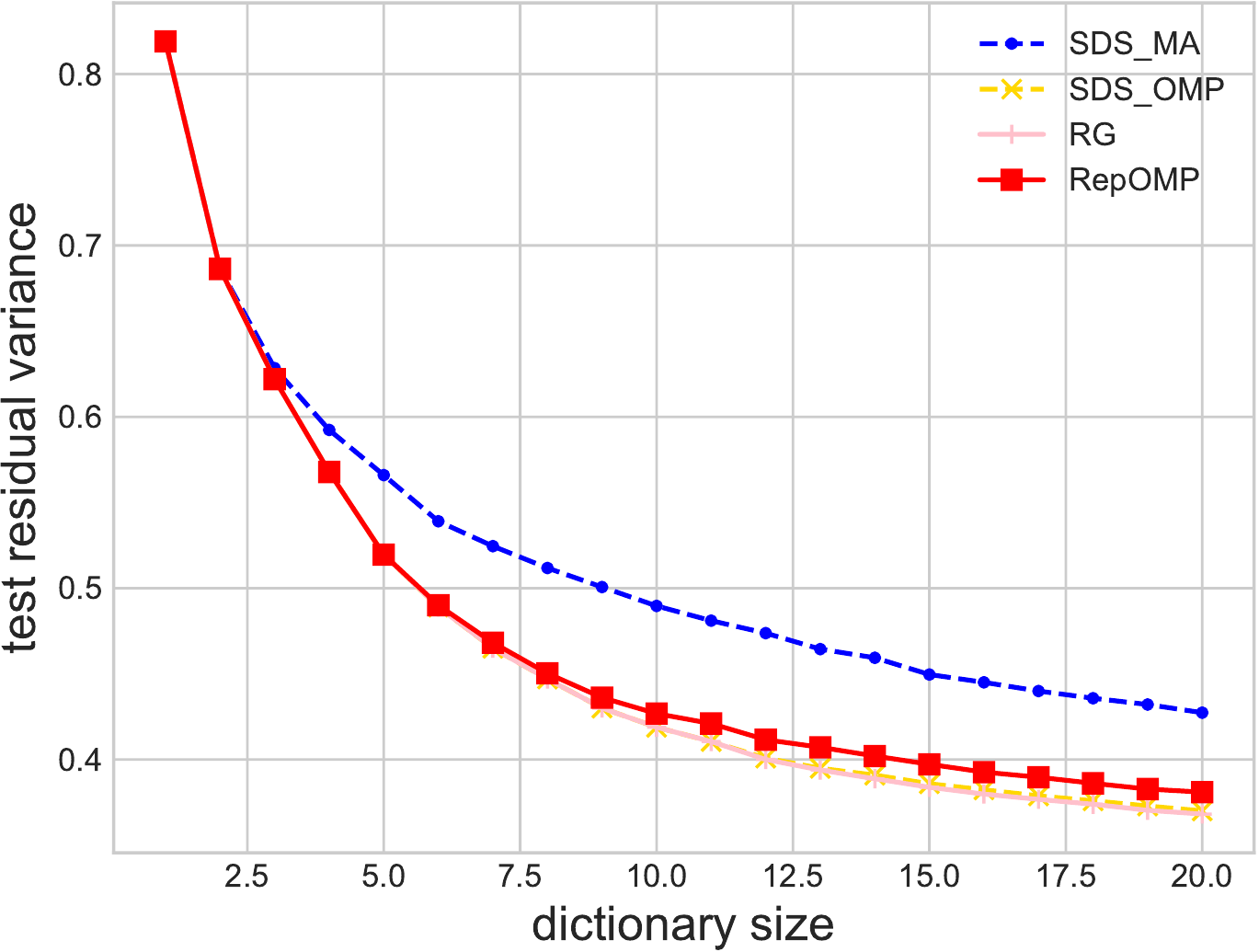}\label{fig:offline_voc_T100_error}
}

\subfigure[synthetic, $T = 1000$, time]{
	\includegraphics[width=0.31\textwidth]{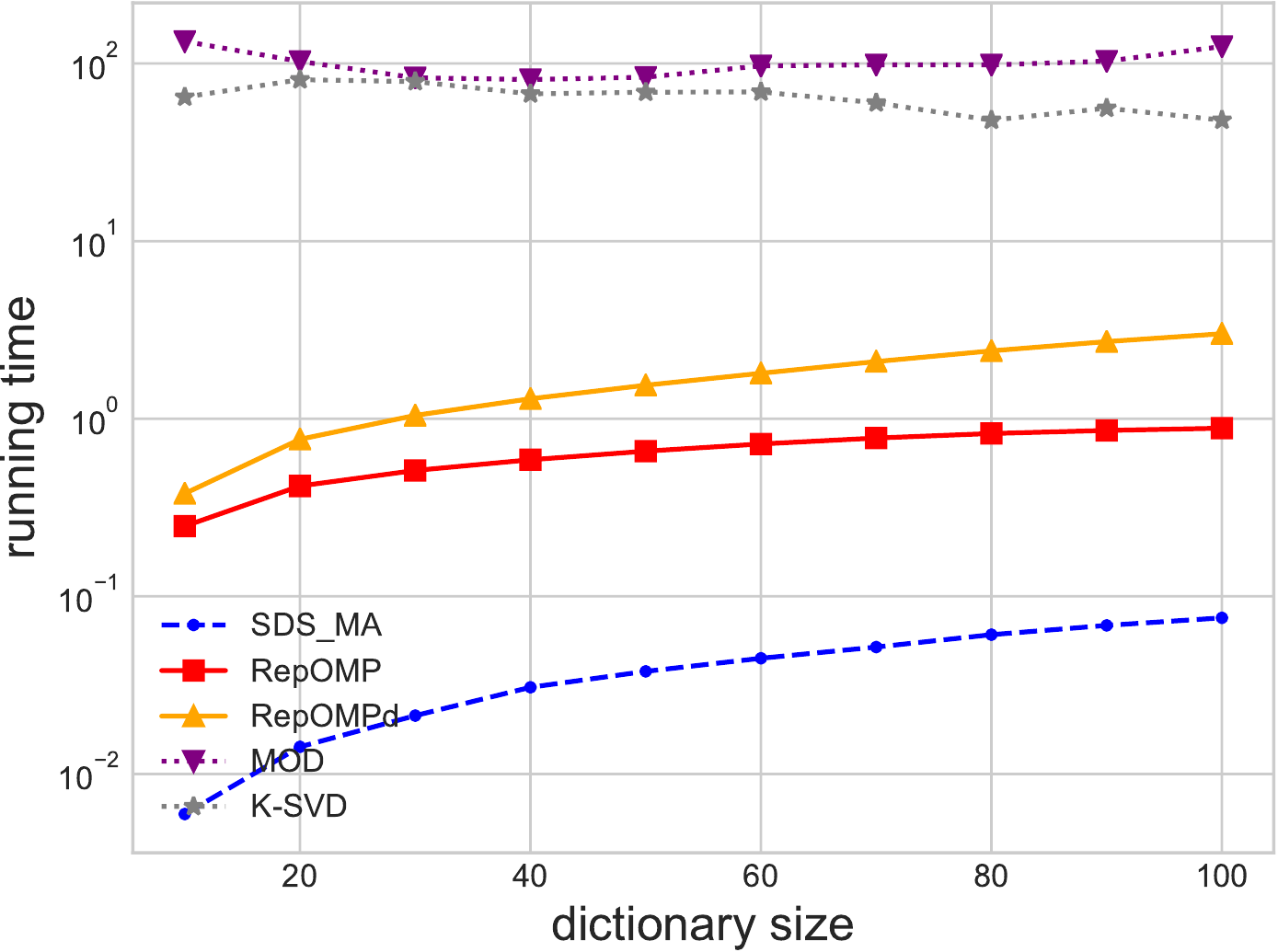}\label{fig:offline_synthetic_T1000_time}
}
\subfigure[synthetic, $T = 1000$, residual]{
	\includegraphics[width=0.31\textwidth]{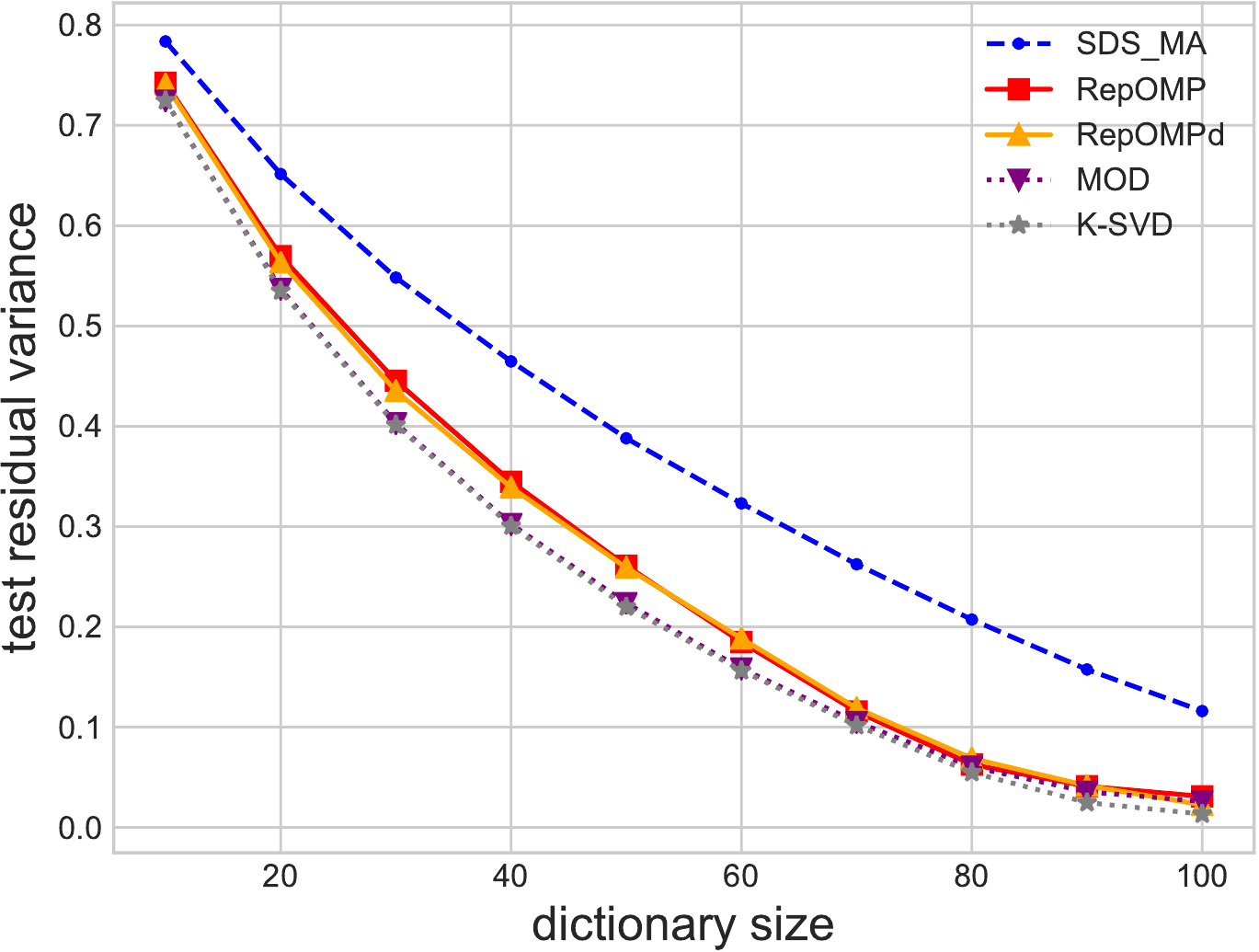}\label{fig:offline_synthetic_T1000_error}
}
\subfigure[voc, $T = 1000$, residual]{
	\includegraphics[width=0.31\textwidth]{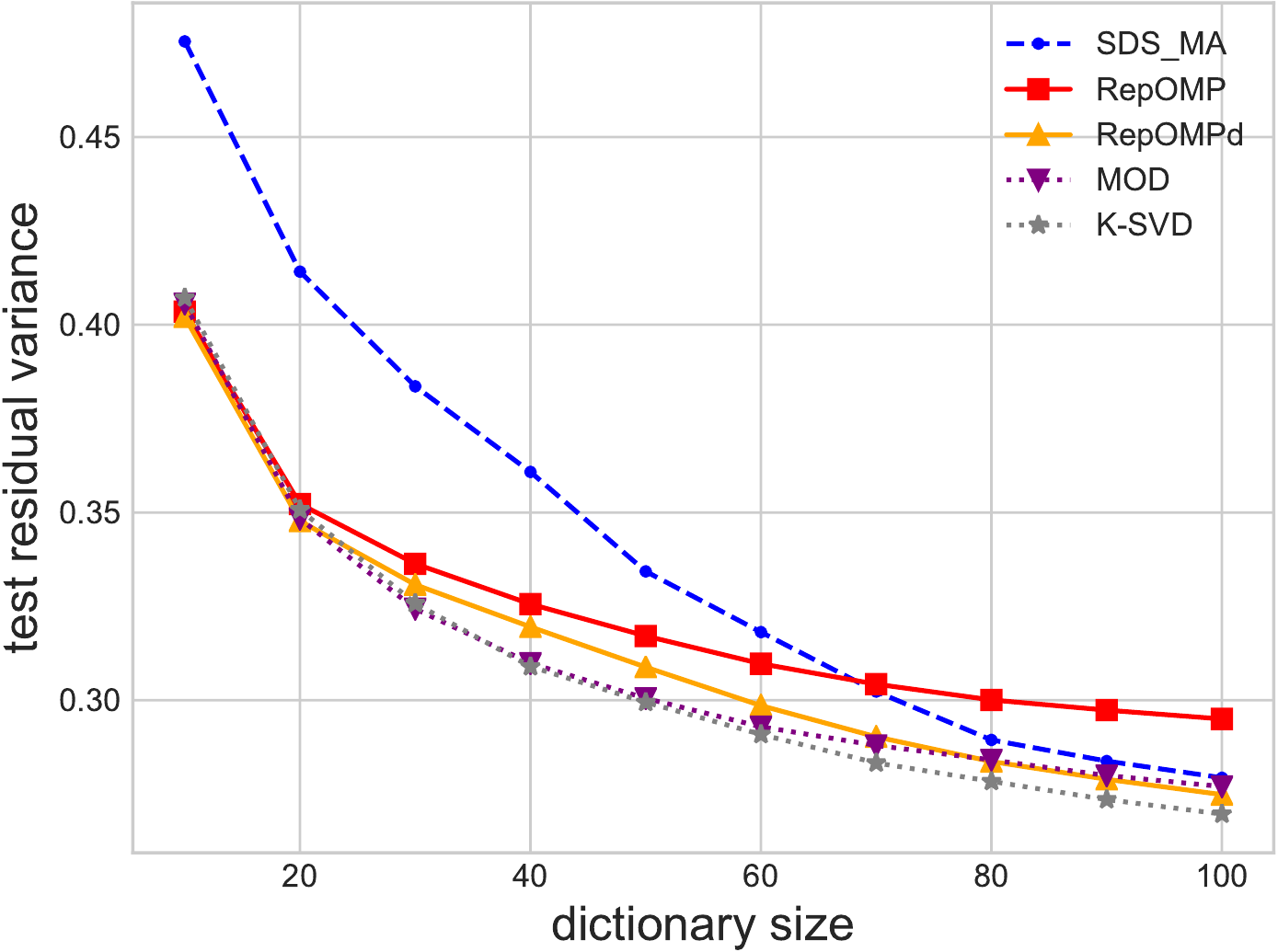}\label{fig:offline_voc_T1000_error}
}

\subfigure[synthetic, $T = 1000$, time]{
	\includegraphics[width=0.31\textwidth]{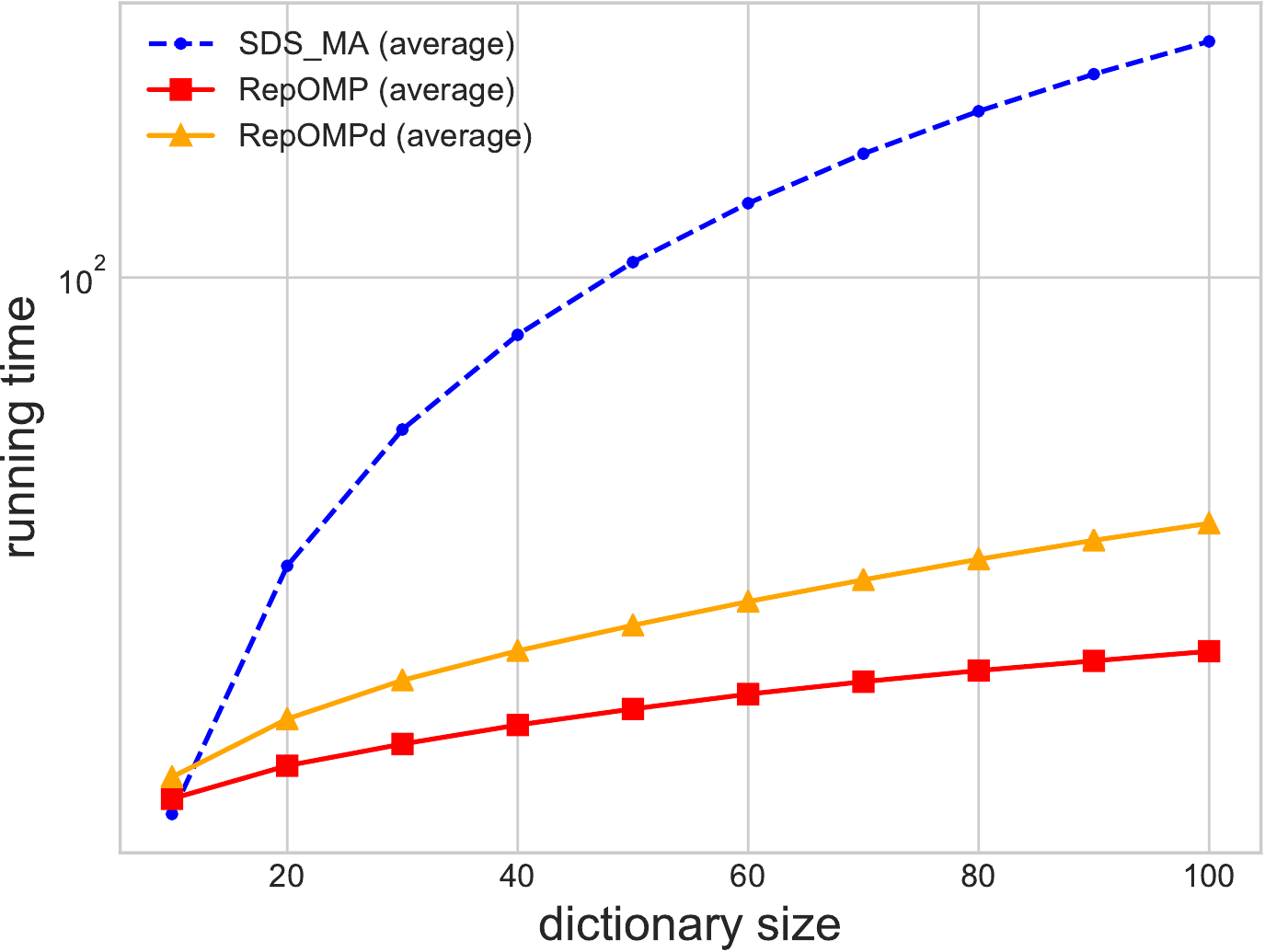}\label{fig:offline_synthetic_average_T1000_time}
}
\subfigure[synthetic, $T = 1000$, residual]{
	\includegraphics[width=0.31\textwidth]{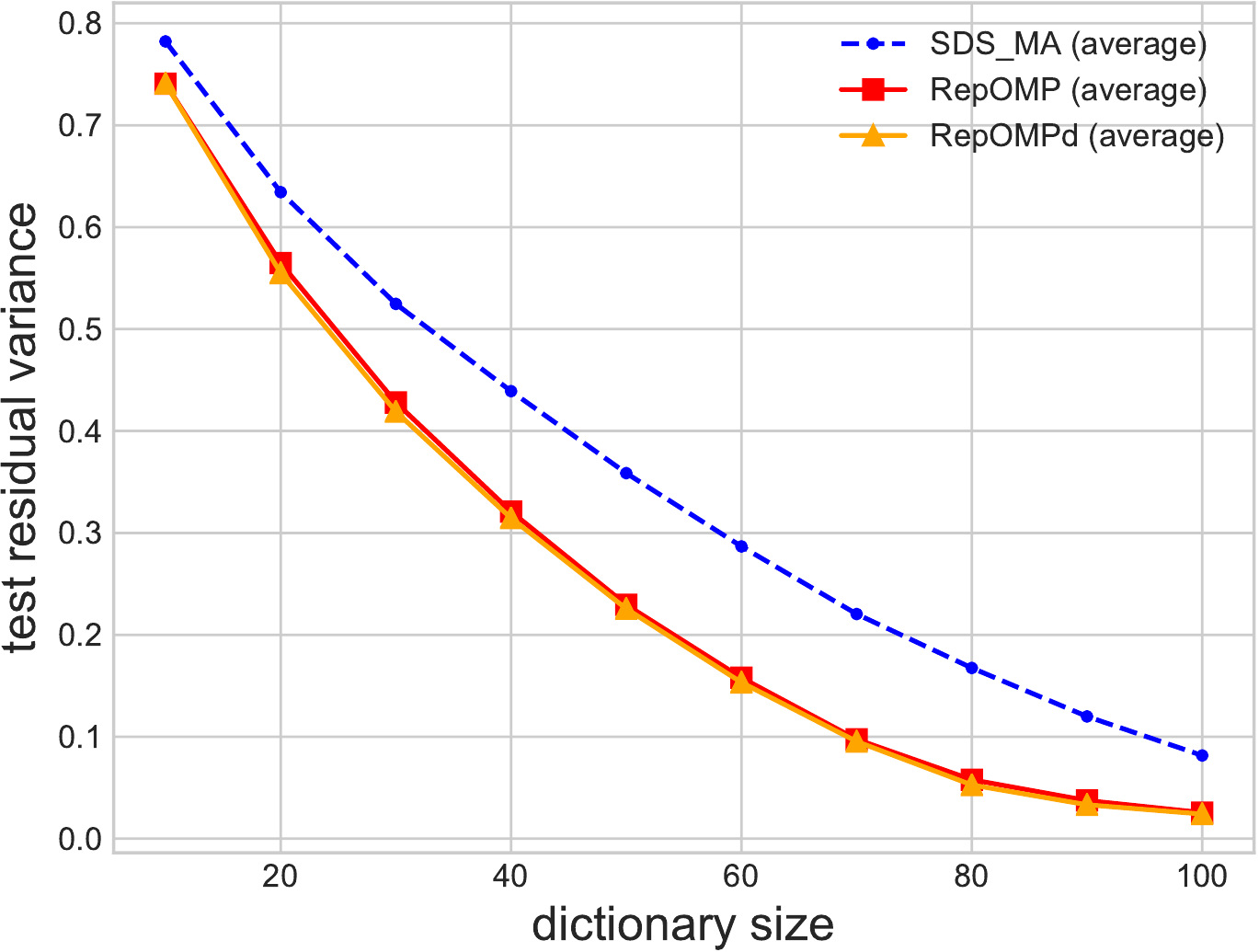}\label{fig:offline_synthetic_average_T1000_error}
}
\subfigure[voc, $T = 1000$, residual]{
	\includegraphics[width=0.31\textwidth]{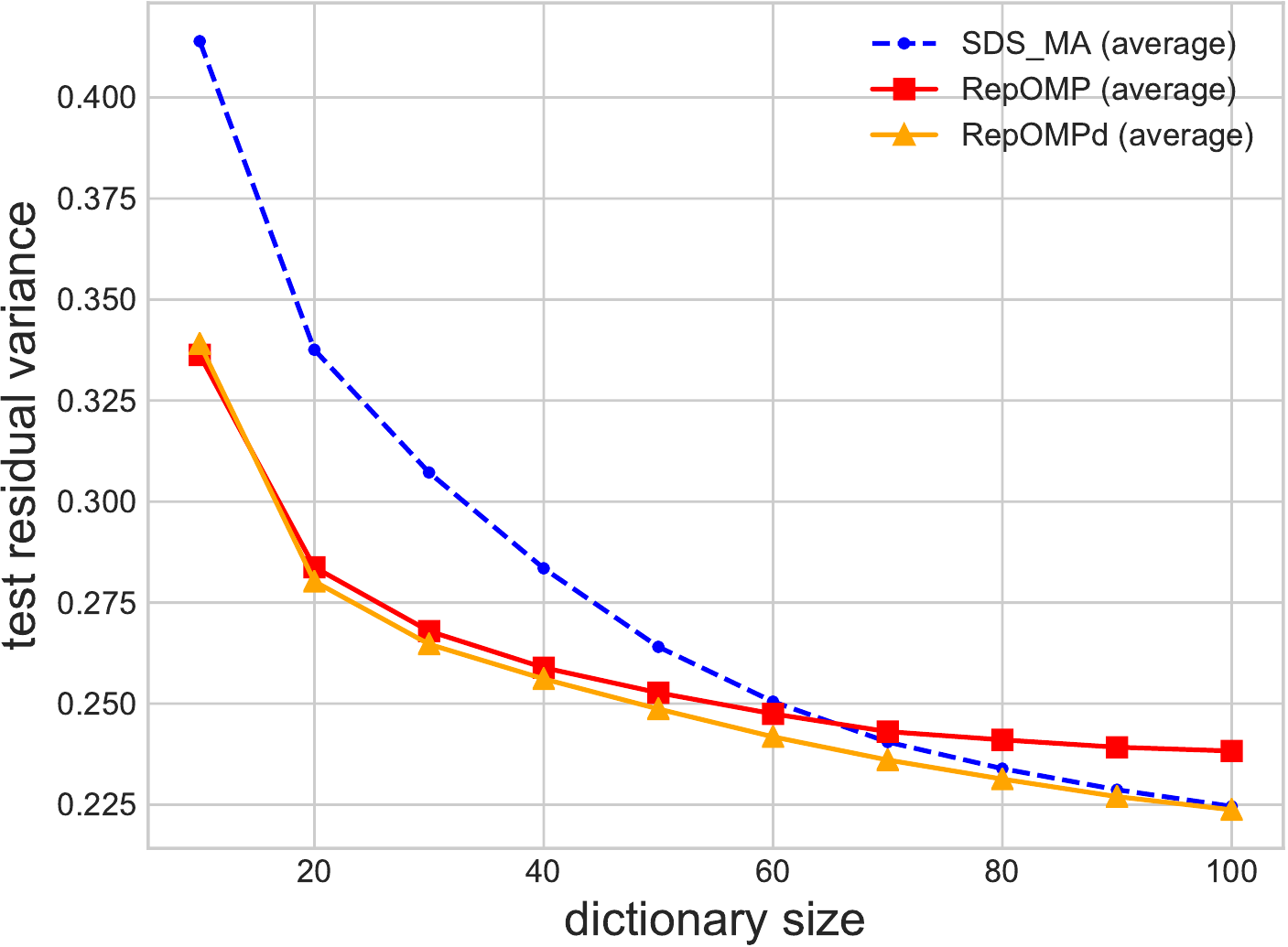}\label{fig:offline_voc_average_T1000_error}
}
\caption{The experimental results for the offline setting. In all figures, the horizontal axis indicates the size of the output dictionary. 
    \subref{fig:offline_synthetic_T100_time}, \subref{fig:offline_synthetic_T100_error}, and \subref{fig:offline_voc_T100_error} are the results for $T = 100$. \subref{fig:offline_synthetic_T1000_time}, \subref{fig:offline_synthetic_T1000_error}, and \subref{fig:offline_voc_T1000_error} are the results for $T = 1000$. \subref{fig:offline_synthetic_average_T1000_time}, \subref{fig:offline_synthetic_average_T1000_error}, and \subref{fig:offline_voc_average_T1000_error} are the results for $T = 1000$ with an average sparsity constraint. For each setting, we provide the plot of the running time for the synthetic dataset, test residual variance for the synthetic dataset, and test residual variance for VOC2006 image dataset.}
\end{figure*}

In this section, we empirically evaluate our proposed algorithms on several dictionary selection problems with synthetic and real-world datasets.
We use the squared $\ell^2$-utility function 
for all of the experiments. 
Since evaluating the value of the objective function is NP-hard, we plot the approximated residual variance obtained by orthogonal matching pursuit.

\paragraph{Ground set}
We use the ground set consisting of several orthonormal bases that are standard choices in signal and image processing, such as 2D discrete cosine transform and several 2D discrete wavelet transforms (Haar, Daubechies 4, and coiflet). 
In all of the experiments, the dimension is set to $d = 64$, which corresponds to images of size $8 \times 8$ pixels.
The size of the ground set is $n = 256$.

\paragraph{Machine}
All the algorithms are implemented in Python 3.6.
We conduct the experiments in a machine with Intel Xeon E3-1225 V2 (3.20 GHz and 4 cores) and 16 GB RAM.

\paragraph{Datasets}
We conduct experiments on two types of datasets. 
The first one is a synthetic dataset. 
In each trial, we randomly pick a dictionary with size $k$ out of the ground set, and generate sparse linear combinations of the columns of this dictionary. 
The weights of the linear combinations are generated from the standard normal distribution. 
The second one is a dataset of real-world images extracted from PASCAL VOC2006 image datasets \citep{pascal-voc-2006}. 
In each trial, we randomly select an image out of 2618 images and divide it into patches of $8 \times 8$ pixels, then select $T$ patches uniformly at random. 
All the patches are normalized to zero mean and unit variance.
We make datasets for training and test in the same way, and use the training dataset for obtaining a dictionary and the test dataset for measuring the quality of the output dictionary.

\subsection{Experiments on the offline setting}
We implement our proposed methods, \ReplacementGreedy{} (\textsf{RG}) and \ReplacementOMP{} (\textsf{ROMP}), as well as the existing methods for dictionary selection, \ModularApproximation{} and \OMPEvaluation{}.
We also implement a heuristically modified version of \textsf{ROMP}, which we call \textsf{ROMPd}.  
In \textsf{ROMPd}, we replace $M_{s,2}$ with some parameter that decreases as the size of the current dictionary grows, which prevents the gains of all the atoms from being zero. Here we use $M_{s,2} / \sqrt{i}$ as the decreasing parameter where $i$ is the number of iterations so far.
In addition, we compare these methods with standard methods for dictionary learning, \textsf{MOD} \citep{Engan1999} and \textsf{KSVD} \citep{Aharon2006}, which is set to stop when the change of the objective value becomes no more than $10^{-6}$ or 200 iterations are finished. Orthogonal matching pursuit is used as a subroutine in both methods.

First, we compare the methods for dictionary selection with small datasets of $T = 100$.
The parameter of sparsity constraints is set to $s = 5$.
The results averaged over 20 trials are shown in \Cref{fig:offline_synthetic_T100_time}, \subref{fig:offline_synthetic_T100_error}, and \subref{fig:offline_voc_T100_error}. 
The plot of the running time for VOC2006 datasets is omitted as it is much similar to that for synthetic datasets.
In terms of running time, \ModularApproximation{} is the fastest, but the quality of the output dictionary is poor. \textsf{ROMP} is several magnitudes faster than \OMPEvaluation{} and \textsf{RG}, but its quality is almost the same with \OMPEvaluation{} and \textsf{RG}. In \Cref{fig:offline_synthetic_T100_error}, test residual variance of \OMPEvaluation{}, \textsf{RG}, and \textsf{ROMP} are overlapped, and in \Cref{fig:offline_voc_T100_error}, test residual variance of \textsf{ROMP} is slightly worse than that of \OMPEvaluation{} and \textsf{RG}. From these results, we can conclude that \textsf{ROMP} is by far the most practical method for dictionary selection.

Next we compare the dictionary selection methods with the dictionary learning methods with larger datasets of $T = 1000$. \OMPEvaluation{} and \textsf{RG} are omitted because they are too slow to be applied to datasets of this size. The results averaged over 20 trials are shown in \Cref{fig:offline_synthetic_T1000_time}, \subref{fig:offline_synthetic_T1000_error}, and \subref{fig:offline_voc_T1000_error}. In terms of running time, \textsf{ROMP} and \textsf{ROMPd} are much faster than \textsf{MOD} and \textsf{KSVD}, but their performances are competitive with \textsf{MOD} and \textsf{KSVD}.

Finally, we conduct experiments with the average sparsity constraints. We compare \textsf{ROMP} and \textsf{ROMPd} with \Cref{alg:greedy-selection-average-sparsity} in \Cref{sec:alg-appendix} with a variant of \ModularApproximation{} proposed for average sparsity in \citet{Cevher2011}. The parameters of constraints are set to $s_t = 8$ for all $t \in [T]$ and $s' = 5T$. The results averaged over 20 trials are shown in \Cref{fig:offline_synthetic_average_T1000_time}, \subref{fig:offline_synthetic_average_T1000_error}, and \subref{fig:offline_voc_average_T1000_error}. \textsf{ROMP} and \textsf{ROMPd} outperform \ModularApproximation{} both in running time and quality of the output.

In \Cref{sec:further}, We provide further experimental results. There we provide examples of image restoration, in which the average sparsity works better than the standard dictionary selection.

\subsection{Experiments on the online setting}
Here we give the experimental results on the online setting. We implement the online version of \ModularApproximation{}, \textsf{RG} and \textsf{ROMP}, as well as an online dictionary learning algorithm proposed by \citet{Mairal2010}. For all the online dictionary selection methods, the hedge algorithm is used as the subroutines.
The parameters are set to $k = 20$ and $s = 5$.
The results averaged over 50 trials are shown in \Cref{fig:online_synthetic_k20}, \subref{fig:online_voc_k20}. 
For both datasets, Online \textsf{ROMP} shows a better performance than Online \ModularApproximation{}, Online \textsf{RG}, and the online dictionary learning algorithm.

\begin{figure*}
\centering
\subfigure[synthetic]{
	\includegraphics[width=0.31\textwidth]{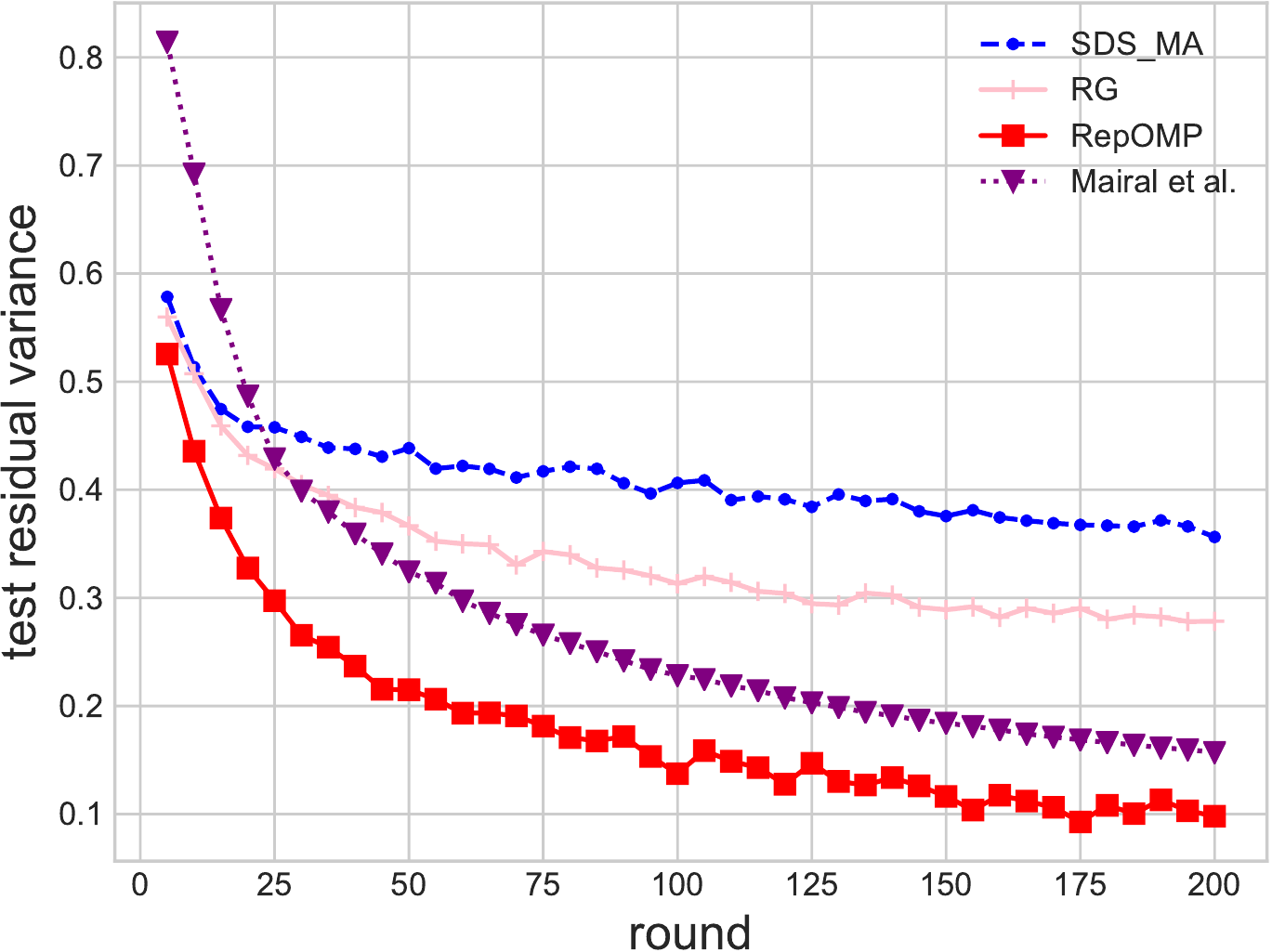}\label{fig:online_synthetic_k20}
}
\subfigure[voc]{
	\includegraphics[width=0.31\textwidth]{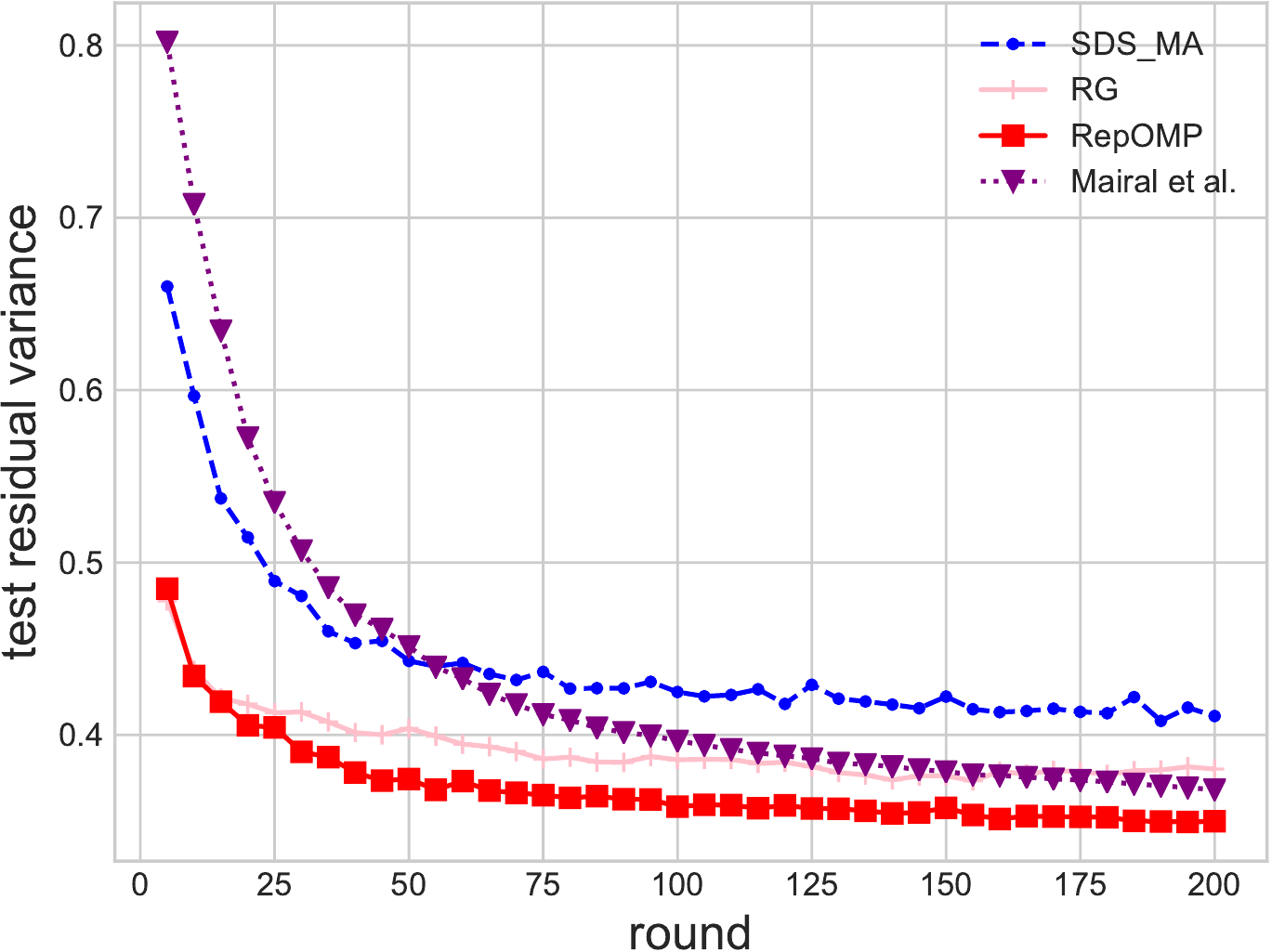}\label{fig:online_voc_k20}
}
\caption{The experimental results for the online setting. In both figures, the horizontal axis indicates the number of rounds. 
    \subref{fig:online_synthetic_k20} is the result with synthetic datasets, and \subref{fig:online_voc_k20} is the result with VOC2006 image datasets.}
\end{figure*}

\newpage

\section*{Acknowledgement}
The authors would thank Taihei Oki and Nobutaka Shimizu for their stimulating discussions.
K.F. was supported by JSPS KAKENHI Grant Number JP 18J12405. T.S. was supported by ACT-I, JST.
This work was supported by JST CREST, Grant Number JPMJCR14D2, Japan. 

\bibliographystyle{abbrvnat}
\bibliography{main}

\newpage
\appendix
\onecolumn

\section{Miscellaneous fact}
The following folklore result is often useful for proving an approximate ratio.

\begin{lemma}\label{lem:misc}
    Suppose that $\Delta_i, r_i \geq 0$ ($i = 1, 2, \dots$) satisfies
    \begin{align}
        \Delta_i \geq C \left(\OPT - \sum_{j = 1}^{i-1} \Delta_j \right) - r_i,
    \end{align}
    for $i = 1, 2, \dots$, for some constants $C \in [0,1]$ and $\OPT \geq 0$.
    Then
    \begin{align}
        \sum_{i=1}^l \Delta_i \geq \left[1 - (1 - C)^l \right] \OPT - \sum_{i=1}^l r_i \geq (1- \exp(-Cl)) \OPT - \sum_{i=1}^l r_i.
    \end{align}
    for any nonnegative integer $l$.
\end{lemma}
\begin{proof}
    We show 
    \begin{align}\label{eq:misc}
        \OPT - \sum_{i=1}^l \Delta_i \leq \left(1 - C \right)^l \OPT + \sum_{i=1}^l r_i 
    \end{align}
    for $l = 0, 1, 2, \dots$ by the induction on $l$.
    For $l = 0$, \eqref{eq:misc} is trivial.
    For $l\geq 1$, we have
    \begin{align*}
        \OPT - \sum_{i=1}^l \Delta_i 
        &= \OPT - \sum_{i=1}^{l-1} \Delta_i - \Delta_l \\
        &\leq \OPT - \sum_{i=1}^{l-1} \Delta_i - C\left(\OPT - \sum_{j = 1}^{l-1} \Delta_j \right) + r_l \\
        &= \left(1 - C \right)\left(\OPT - \sum_{i=1}^{l-1} \Delta_i \right) + r_l.
    \end{align*}
    Now \eqref{eq:misc} follows from the induction and $1 - C \in [0, 1]$.
\end{proof}

\section{Missing proofs for generalized sparsity constraints}
\label{sec:general-appendix}

\subsection{Individual matroids}
\begin{proposition}\label{prop:matroid-sparsity}
    An individual matroid constraint is $k$-replacement sparse.
\end{proposition}

\begin{proof}
    Let $(Z_1,\cdots,Z_T), (Z^*_1,\cdots,Z^*_T) \in \mathcal{I}$ be arbitrary sparse subsets. First we consider the case where $Z_t$ and $Z^*_t$ are both bases\footnote{For any matroid $(V, \mathcal{I})$, a set $X \in \mathcal{I}$ is called \textit{a base} if it is maximal in $\mathcal{I}$.} of the matroid for all $t \in [T]$. For such $Z_t$ and $Z^*_t$, we can make $k$ replacements as follows: For each $t \in [T]$, there exists a bijection $\pi_t \colon Z^*_t \to Z_t$ by the exchange property of matroids. For each atom $a^* \in \bigcup_{t=1}^T Z^*_t$, we make a replacement that adds $a^*$ to and removes $\pi_t(a^*)$ from $Z_t$ for all $t \in [T]$ such that $a^* \in Z^*_t$.

	If $Z_t$ or $Z^*_t$ is not a base of the matroid, we can add arbitrary atoms to $Z_t$ and $Z^*_t$ until they are both bases, and make $k$ replacements for them in the same way as described above. Removing the atoms that do not exist in $Z_t$ and $Z^*_t$ from these $k$ replacements, we obtain replacements for original $Z_t$ and $Z^*_t$.
\end{proof}

\subsection{Block sparsity}
\begin{proposition}
    A block sparsity constraint is $k$-replacement sparse.
\end{proposition}

\begin{proof}
	Let $(Z_1,\cdots,Z_T), (Z^*_1,\cdots,Z^*_T) \in \mathcal{I}$ be arbitrary sparse subsets. We can make $k$ replacements as follows: Let $Z_{b'} = \bigcup_{t \in B_{b'}} Z_t$ and $Z^*_{b'} = \bigcup_{t \in B_{b'}} Z^*_t$. If $|Z_{b'}| < s_{b'}$ or $|Z_{b'}| < s_{b'}$, we can add arbitrary atoms until these inequalities are tight. For each block $b' \in [b]$, we can make a bijection $\pi_t \colon Z^*_b \to Z_b$. For each atom $a^* \in \bigcup_{t=1}^T Z^*_t$, we make one replacement that adds $a^*$ for all $t \in [T]$ such that $a^* \in Z^*_t$ and removes $\pi_t(a^*)$ from all blocks such that $a^* \in \bigcup_{t \in B_{b'}} Z^*_t$.
\end{proof}
We can show the common generalization of an individual matroid sparsity and block sparsity is also $k$-replacement sparse by combining the proofs.

\subsection{Average sparsity without individual sparsity}
%
First we consider an easier case with only a total number constraint, that is, $\mathcal{I} = \{(Z_1,\cdots,Z_T) \mid \sum_{t = 1}^T |Z_t| \le s' \}$. We call it an average sparsity constraint without individual sparsity.
\begin{proposition}\label{prop:average-sparsity-replacement}
	An average sparsity constraint without individual sparsity is $(2k-1)$-replacement sparse.
\end{proposition}

\begin{proof}
    Let $(Z_1,\cdots,Z_T), (Z^*_1,\cdots,Z^*_T) \in \mathcal{I}$ be arbitrary feasible sparse subsets. We assume $(Z_1,\cdots,Z_T)$ and $(Z^*_1,\cdots,Z^*_T)$ are maximal in $\mathcal{I}$, but we can deal with non-maximal ones by filling them with dummy elements in the same way as the proof of \Cref{prop:matroid-sparsity}. Here we show it is possible to greedily make a sequence of $2k - 1$ feasible replacements $(Z^{r'}_1,\cdots,Z^{r'}_T)_{r'=1}^{2k-1}$ such that each atom in $Z^*_t \setminus Z_t$ appears at least once in the sequence $(Z^{r'}_t \setminus Z_t)_{r'=1}^{2k-1}$ and each atom in $Z_t \setminus Z_t^*$ appears at most once in the sequence $(Z_t \setminus Z^{r'}_t)_{r'=1}^{2k-1}$.

    Let $X$ and $X^*$ be the sets of atoms appearing in $(Z_1, \cdots, Z_T)$ and $(Z^*_1, \cdots, Z^*_T)$, respectively.
    We arrange the atoms in each of $X$ and $X^*$ in an arbitrary order and consider them one by one in parallel. Let us suppose we currently consider $a \in X$ and $a^* \in X^*$. We make a replacement that adds $a^*$ for several $t \in [T]$ and removes $a$ for the other several $t \in [T]$ in the following way. Let $\tau$ be the number of $Z_t \setminus Z^*_t$ that contains $a$, i.e., $\tau = |\{t \in [T] \mid a \in Z_t \setminus Z^*_t\}|$ and $\tau^*$ the number of $Z^*_t \setminus Z_t$ that contains $a^*$, i.e., $\tau = |\{t \in [T] \mid a^* \in Z^*_t \setminus Z_t\}|$. If $\tau > \tau^*$, we can let this replacement add $a^*$ for all $t \in [T]$ such that $a^* \in Z^*_t \setminus Z_t$ and remove $a$ for any subset of $\{t \in [T] \mid a \in Z_t \setminus Z^*_t\}$ with size $\tau^*$. Conversely, if $\tau \le \tau^*$, we can let this replacement add $a^*$ for an arbitrary subset of $\{t \in [T] \mid a^* \in Z^*_t \setminus Z_t\}$ of size $\tau$ and remove $a$ for all $t \in [T]$ such that $a \in Z_t \setminus Z^*_t$. We proceed to a next replacement after removing $a^*$ from $Z^*_t$ for all $t \in [T]$ such that $a^*$ is added in this replacement, and $a$ from $Z_t$ for all $t \in [T]$ such that $a$ is removed in this replacement. If $a \not\in Z_t \setminus Z^*_t$ for all $t \in [T]$, we move the focus from $a$ to the next atom. Similarly, if $a^* \not\in Z^*_t \setminus Z_t$ for all $t \in [T]$, we move the focus from $a$ to the next atom.

    This procedure ends after at most $2k - 1$ iterations. This is because at each iteration we move the focus from $a$ to the next atom in $X$ or from $a^*$ to the next atom in $X^*$, and we obtain $|X| \le k$ and $|X^*| \le k$.
\end{proof}

Here we show this bound is tight for an average sparsity constraint without individual sparsity by giving an example.
\begin{example}
    Assume $T \ge k^2$. For simplicity, we further assume $T$ is a multiple of $k$. Let us consider the case of $s' = T$, i.e., $\mathcal{I} = \{(Z_1,\cdots,Z_T) \mid \sum_{t = 1}^T |Z_t| \le T\}$. Let $V = \{v_1,\cdots,v_{2k}\}$ be the ground set. Here we show the replacement sparsity parameter of this sparsity constraint is at least $2k - 1$ by giving $(Z_1,\cdots,Z_T)$ and $(Z^*_1,\cdots,Z^*_T)$ that require $2k - 1$ replacements. Suppose $Z_t = \{v_1,\cdots,v_k\}$ for $1 \le t \le T / k$ and $Z_t = \emptyset$ for other $t$. Let $Z^*_t = \{v_{k+1}\}$ for $1 \le t \le T - k + 1$ and $Z^*_{T - k + i} = \{v_{k+i}\}$ for each $i = 2, \cdots, k$.  .
	
    It can be seen that we must use $k - 1$ different replacements for $Z^*_{T-k+2}, \cdots, Z^*_{T}$. In each replacement, an added element is restricted to a single atom, but $Z^*_{T-k+2}, \cdots, Z^*_{T}$ are all singleton sets of different atoms. Then elements in $Z^*_{T-k+2}, \cdots, Z^*_{T}$ must be dealt with by different replacements, and $k-1$ replacements are needed.
	
    In addition, we must use $k$ other replacements for $Z^*_1,\cdots,Z^*_{T-k+1}$. Since $(Z_1,\cdots,Z_T)$ is maximal in $\mathcal{I}$, the total number of added atoms of each replacement must be at most the total number of removed atoms of this replacement. However, in each replacement, the number of atoms removed from each $Z_t$ is at most one, and only $Z_1,\cdots,Z_{T/k}$ are non-empty, hence at most $T / k$ elements can be removed in each replacement. Therefore, we must use $k$ different replacements for $Z^*_1,\cdots,Z^*_{T-k+1}$ because there are $T - k + 1$ singleton sets $Z^*_1,\cdots,Z^*_{T-k+1}$ and $T \ge k^2$.

	In conclusion, the replacement sparsity parameter of this sparsity constraint is at least $2k - 1$.
\end{example}

\subsection{Average sparsity}
We bound the replacement sparsity parameter of an average sparsity constraint based on the analysis on average sparsity without individual sparsity.
\begin{proposition}
	An average sparsity constraint is $(3k-1)$-replacement sparse.
\end{proposition}

\begin{proof}
	Here we give a sequence of $3k - 1$ replacements that satisfies the conditions for replacement sparsity.
    
    First we use $k$ replacements for dealing with the individual sparsity constraints. Let $S \subseteq [T]$ be the set of indices such that $|Z_t| = s_t$. For each $a^* \in X^*$, we make a replacement that adds $a^*$ for all $t \in S$ such that $a^* \in Z^*_t \setminus Z_t$ and possibly removes an atom in $Z_t \setminus Z^*_t$ for all $t \in S$. By selecting the removed atoms so that they do not overlap, we can define these $k$ replacements such that, for all $t \in S$, each atom in $Z^*_t \setminus Z_t$ is added once and each atom in $Z_t \setminus Z^*_t$ is removed once.

    For the rest of the elements, we need not consider the individual sparsity constraints, therefore the rest elements can be dealt with $2k-1$ replacements in the same way as the proof of \Cref{prop:average-sparsity-replacement}.
\end{proof}

\section{Proofs for \ReplacementGreedy{} and \ReplacementOMP{}}\label{sec:alg-appendix}

\subsection{Proof for \ReplacementGreedy{}}
\begin{lemma}\label{lem:replacement-greedy-marginal}
    Assume $\mathcal{I}$ is $p$-replacement sparse.
    Suppose that at some step, the solution is updated from $(Z_1,\cdots,Z_T)$ to $(Z'_1,\cdots,Z'_T)$ by \ReplacementGreedy{}.
    Let $(Z^*_1,\cdots,Z^*_T) \in \argmax_{(Z_1,\cdots,Z_T) \in \mathcal{I} \colon Z_t \subseteq X^*} f_t(Z)$ where $X^*$ is an optimal solution for dictionary selection.
    Then, the marginal gain of \ReplacementGreedy{} is bounded from below as follows:
    \begin{equation*}
        \sum_{t = 1}^T f_t(Z'_t) - \sum_{t = 1}^T f_t(Z_t) \ge \frac{1}{p} \left\{ \frac{m_{2s}}{M_{s,2}} \sum_{t = 1}^T f_t(Z^*_t) - \frac{M_{s,2}}{m_{2s}} \sum_{t = 1}^T f_t(Z_t) \right\}
    \end{equation*}
    where $s = \max_{(Z_t)_{t=1}^T \in \mathcal{I}} \max_{t \in [T]} |Z_t|$.
\end{lemma}

\begin{proof}
    Note that from the condition on feasible replacements, we have $|Z_t \triangle Z'_t| \le 2$. Since $u_t$ is $M_{s,2}$-smooth on $\Omega_{s,2}$, it holds that for any $\bz \in \R^n$ with $\supp(\bz) \subseteq Z'_t \setminus Z_t$, 
    \begin{align*}
        f_t(Z'_t) - f_t(Z_t)
        &= u_t (\bw^{(Z'_t)}) - u_t (\bw^{(Z_t)})\\
        &\ge u_t ((\bw^{(Z_t)})_{Z_t \cap Z'_t} + \bz) - u_t (\bw^{(Z_t)})\\
        &\ge \left\langle \nabla u_t (\bw^{(Z_t)}), \bz - (\bw^{(Z_t)})_{Z_t \setminus Z'_t} \right\rangle - \frac{M_{s,2}}{2} \norm{ \bz - (\bw^{(Z_t)})_{Z_t \setminus Z'_t}}^2
    \end{align*}
    Since this inequality holds for every $\bz$ with $\supp(\bz) \subseteq Z'_t \setminus Z_t$, by optimizing it for $\bz$, we obtain
    \begin{equation}\label{eq:replacement-greedy-general-ineq1}
        f_t(Z'_t) - f_t(Z_t)
        \ge \frac{1}{2M_{s,2}} \norm{\nabla u_t(\bw_t^{(Z_t)})_{Z'_t \setminus Z_t}}^2 - \frac{M_{s,2}}{2} \norm{(\bw_t^{(Z_t)})_{Z_t \setminus Z'_t} }^2.
    \end{equation}
    In addition, due to the strong concavity of $u_t$, we have
    \begin{align}\label{eq:replacement-greedy-general-ineq2}
        f_t(Z_t^*) - f_t(Z_t)
        &= u_t(\bw^{(Z_t^*)}) - u_t(\bw^{(Z_t)}) \nonumber\\
        &\le \left\langle \nabla u_t (\bw_t^{(Z_t)}), \bw_t^{(Z^*_t)} - \bw_t^{(Z_t)} \right\rangle - \frac{m_{2s}}{2} \left\| \bw_t^{(Z_t^*)} - \bw_t^{(Z_t)} \right\|^2 \nonumber \\
        &\le \max_{\bz \colon \supp(\bz) \subseteq Z_t^*} \left\{ \left\langle \nabla u_t (\bw_t^{(Z_t)}), \bz - \bw_t^{(Z_t)} \right\rangle - \frac{m_{2s}}{2} \left\| \bz - \bw_t^{(Z_t)} \right\|^2 \right\} \nonumber \\
        &= \frac{1}{2m_{2s}} \left\| (\nabla u_t (\bw^{(Z_t)}))_{Z_t^* \setminus Z_t} \right\|^2 - \frac{m_{2s}}{2} \left\| (\bw^{(Z_t)})_{Z_t \setminus Z_t^*} \right\|^2.
    \end{align}

    Similarly, due to the strong concavity of $u_t$, we have
    \begin{align}\label{eq:replacement-greedy-general-ineq3}
        - f_t(Z_t)
        &= u_t(\bfzero) - u_t(\bw_t^{(Z_t)}) \nonumber \\
        &\le \left\langle \nabla u_t (\bw_t^{(Z_t)}), - \bw_t^{(Z_t)} \right\rangle - \frac{m_{2s}}{2} \left\| \bw_t^{(Z_t)} \right\|^2 \nonumber \\
        &= - \frac{m_{2s}}{2} \left\| \bw_t^{(Z_t)} \right\|^2 \nonumber \\
        &\le - \frac{m_{2s}}{2} \left\| (\bw_t^{(Z_t)})_{Z_t \setminus Z^*_t} \right\|^2
    \end{align}
    Since $\mathcal{I}$ is $p$-replacement sparse, we can take a sequence of $p$ replacements $(Z^{p'}_1, \cdots, Z^{p'}_T)_{p'=1}^p$ such that
    \begin{itemize}
        \item $(Z^{p'}_1, \cdots, Z^{p'}_T) \in \mathcal{F}(Z_1, \cdots, Z_T)$,
        \item each element in $Z^*_t \setminus Z_t$ appears at least once in sequence $(Z^{p'}_t \setminus Z_t)_{p'=1}^p$ for each $t \in [T]$,
        \item each element in $Z_t \setminus Z_t^*$ appears at most once in sequence $(Z_t \setminus Z^{p'}_t)_{p'=1}^p$ for each $t \in [T]$.
    \end{itemize}

    Now we prove the lemma by utilizing these properties.
    \begin{align*}
        &\sum_{t = 1}^T f_t(Z'_t) - \sum_{t = 1}^T f_t(Z_t) \\
        &\ge \frac{1}{p} \sum_{p' = 1}^{p} \left\{ \sum_{t = 1}^T f_t(Z^{p'}_t) - \sum_{t = 1}^T f_t(Z_t) \right\}
        \tag{by the choice of $(Z_1', \dots, Z_T')$ and the feasibility of $(Z^{p'}_1,\cdots,Z^{p'}_T)$} 
        \\
        &\ge \frac{1}{p} \sum_{p'=1}^p \sum_{t = 1}^T \left\{ \frac{1}{2M_{s,2}} \norm*{\nabla u_t(\bw_t^{(Z_t)})_{Z^{p'}_t \setminus Z_t}}^2 - \frac{M_{s,2}}{2} \norm*{\left( \bw_t^{(Z_t)} \right)_{Z_t \setminus Z^{p'}_t}}^2 \right\}
        \tag{by \eqref{eq:replacement-greedy-general-ineq1}}
        \\
    &\ge \frac{1}{p} \sum_{t = 1}^T \left\{ \frac{1}{2M_{s,2}} \norm*{\nabla u_t(\bw_t^{(Z_t)})_{Z^*_t \setminus Z_t}}^2 - \frac{M_{s,2}}{2} \norm*{\left( \bw_t^{(Z_t)} \right)_{Z_t \setminus Z_t^*}}^2 \right\}
        \tag{by the property of $(Z_t^{p'})_{p'=1}^p$} \\
        &\ge \frac{1}{p} \sum_{t = 1}^T \left\{ \frac{m_{2s}}{M_{s,2}} \left( f_t(Z_t^*) - f_t(Z_t) \right) - \left( \frac{M_{s,2}}{m_{2s}} - \frac{m_{2s}}{M_{s,2}} \right) f_t(Z_t) \right\} 
    \tag{by \eqref{eq:replacement-greedy-general-ineq2} and \eqref{eq:replacement-greedy-general-ineq3}}
        \\
        &= \frac{1}{p} \sum_{t = 1}^T \left\{ \frac{m_{2s}}{M_{s,2}} f_t(Z_t^{*}) - \frac{M_{s,2}}{m_{2s}} f_t(Z_t) \right\}.
    \end{align*}
\end{proof}

Combined with \Cref{lem:misc}, \Cref{thm:replacement-greedy-1} is obtained.

\subsection{Proof for \ReplacementOMP{}}
\begin{lemma}\label{lem:replacement-omp-marginal}
    Assume $\mathcal{I}$ is $p$-replacement sparse.
    Suppose at some step, the solution is updated from $(Z_1,\cdots,Z_T)$ to $(Z'_1,\cdots,Z'_T)$ by \ReplacementOMP{}.
    Let $(Z^*_1,\cdots,Z^*_T) \in \argmax_{(Z_1,\cdots,Z_T) \in \mathcal{I} \colon Z_t \subseteq X^*} f_t(Z)$ where $X^*$ is an optimal solution for dictionary selection.
    Then, the marginal gain of \ReplacementOMP{} is bounded from below as follows:
    \begin{equation*}
        \sum_{t = 1}^T f_t(Z'_t) - \sum_{t = 1}^T f_t(Z_t) \ge \frac{1}{p} \left\{ \frac{m_{2s}}{M_{s,2}} \sum_{t = 1}^T f_t(Z^*_t) - \frac{M_{s,2}}{m_{2s}} \sum_{t = 1}^T f_t(Z_t) \right\}
    \end{equation*}
    where $s = \max_{(Z_t)_{t=1}^T \in \mathcal{I}} \max_{t \in [T]} |Z_t|$.
\end{lemma}

\begin{proof}
    We can obtain the following inequalities from the strong concavity and smoothness of $u_t$ in the same way as the above proof of \Cref{lem:replacement-greedy-marginal}.
    \begin{align}
        f_t(Z'_t) - f_t(Z_t) &\ge \frac{1}{2M_{s,2}} \norm*{\nabla u_t(\bw_t^{(Z_t)})_{Z'_t \setminus Z_t}}^2 - \frac{M_{s,2}}{2} \norm*{(\bw_t^{(Z_t)})_{Z_t \setminus Z'_t} }^2.\label{eq:replacement-omp-general-ineq1}\\
        f_t(Z_t^*) - f_t(Z_t) &\le \frac{1}{2m_{2s}} \left\| (\nabla u_t (\bw^{(Z_t)}))_{Z_t^* \setminus Z_t} \right\|^2 - \frac{m_{2s}}{2} \left\| (\bw^{(Z_t)})_{Z_t \setminus Z_t^*} \right\|^2.\label{eq:replacement-omp-general-ineq2}\\
        - f_t(Z_t)
        &\le - \frac{m_{2s}}{2} \left\| (\bw_t^{(Z_t)})_{Z_t \setminus Z^*_t} \right\|^2.\label{eq:replacement-omp-general-ineq3}
    \end{align}

    Since $\mathcal{I}$ is $p$-replacement sparse, we can take a sequence of $p$ replacements $(Z^{p'}_1, \cdots, Z^{p'}_T)_{p'=1}^p$ that satisfies the properties mentioned in the proof of \Cref{lem:replacement-greedy-marginal}. With these properties, we obtain
    \begin{align*}
        &\sum_{t = 1}^T f_t(Z'_t) - \sum_{t = 1}^T f_t(Z_t) \\
        &\ge \sum_{t = 1}^T \left\{ \frac{1}{2M_{s,2}}  \norm*{\nabla u_t(\bw_t^{(Z_t)})_{Z'_t \setminus Z_t}}^2 - \frac{M_{s,2}}{2} \norm*{(\bw_t^{(Z_t)})_{Z_t \setminus Z'_t}}^2 \right\} 
        \tag{by \eqref{eq:replacement-omp-general-ineq1}}
        \\
        &\ge \frac{1}{p} \sum_{p'=1}^p \sum_{t = 1}^T \left\{ \frac{1}{2M_{s,2}} \norm*{\nabla u_t(\bw_t^{(Z_t)})_{Z^{p'}_t \setminus Z_t}}^2 - \frac{M_{s,2}}{2} \norm*{\left( \bw_t^{(Z_t)} \right)_{Z_t \setminus Z^{p'}_t}}^2 \right\}
        \tag{by the choice of $(Z_1', \dots, Z_T')$ and the feasibility of $(Z^{p'}_1,\cdots,Z^{p'}_T)$} 
        \\
    &\ge \frac{1}{p} \sum_{t = 1}^T \left\{ \frac{1}{2M_{s,2}} \norm*{\nabla u_t(\bw_t^{(Z_t)})_{Z^*_t \setminus Z_t}}^2 - \frac{M_{s,2}}{2} \norm*{\left( \bw_t^{(Z_t)} \right)_{Z_t \setminus Z_t^*}}^2 \right\}
        \tag{by the property of $(Z_t^{p'})_{p'=1}^p$} \\
        &\ge \frac{1}{p} \sum_{t = 1}^T \left\{ \frac{m_{2s}}{M_{s,2}} \left( f_t(Z_t^*) - f_t(Z_t) \right) - \left( \frac{M_{s,2}}{m_{2s}} - \frac{m_{2s}}{M_{s,2}} \right) f_t(Z_t) \right\} 
    \tag{by \eqref{eq:replacement-omp-general-ineq2} and \eqref{eq:replacement-omp-general-ineq3}}
        \\
        &= \frac{1}{p} \sum_{t = 1}^T \left\{ \frac{m_{2s}}{M_{s,2}} f_t(Z_t^{*}) - \frac{M_{s,2}}{m_{2s}} f_t(Z_t) \right\}.
    \end{align*}
\end{proof}

Combined with \Cref{lem:misc}, we obtain \Cref{thm:replacement-omp-general}.

\subsubsection{About greedy selection at each step}
Next we consider how to find the atom with the largest gain at each step of \ReplacementOMP{} for the average sparsity constraints.

First we show that this task reduces to weighted bipartite matching.
Let us fix an atom $a^*$ because we can simply check all the atoms in $V$.
Let $g_t = (\nabla u_t(\bw^{(Z_t)}))^2_{a^*}$ and $c_t = \min_{a \in Z_t} (\bw^{(Z_t)})^2_{a}$ for each $t \in [T]$. Let $S = \{t \in [T] \mid |Z_t| = s_t\}$ be the set of $t \in [T]$ such that the constraint on $|Z_t|$ is tight.

For each $a^* \in V$, the problem of finding the best replacement can be formulated as follows: 
The goal is to maximize $\sum_{t \in A} g_t - \sum_{t \in B} c_t$ by selecting $A \subseteq [T]$ (the set of indices $t$ such that $a^*$ is added to $Z_t$) and $B \subseteq [T]$ (the set of indices $t$ such that an atom is removed from $Z_t$).
We have two constraints on $A$ and $B$. 
The first constraint is $|A| - |B| \le \theta$ where $\theta = s' - \sum_{t = 1}^T |Z_t|$, derived from the total number constraint $\sum_{t=1}^T |Z_t|$. 
The second constraint is $A \cap S \subseteq B$, derived from the individual constraint $|Z_t| \le s_t$. 
In summary, the formulation as an optimization problem is:
\begin{align*}
    \max_{A,B \subseteq [T]} \quad & \sum_{t \in A} g_t - \sum_{t \in B} c_t\\
    \text{subject to} \quad & |A| - |B| \le \theta\\
    & A \cap S \subseteq B.
\end{align*}
This problem can be regarded as a special case of maximum weight bipartite matching problem. 
Let $U = [T]$ and $V = [T] \cup \{d_1,\cdots,d_\theta\}$ be the set of vertices where $d_1,\cdots,d_\theta$ are dummy elements with zero cost, i.e., $c_{d_i} = 0$ for all $i \in [\theta]$. 
Let $E = \{(t, t) \mid t \in S\} \cup (U \setminus S) \times V$ be the set of edges. The weight of each edge $(\alpha, \beta) \in E$ is defined as $w((\alpha, \beta)) = g_\alpha - c_\beta$. 
Then any matching $M \subseteq E$ in this graph corresponds to a solution $A = \partial M \cap U$ and $B = \partial M \cap V \setminus \{d_1,\cdots,d_\theta\}$ in the above optimization problem. 

\begin{algorithm}
\caption{Calculation of the gain for average sparsity constraints}\label{alg:greedy-selection-average-sparsity}
    \begin{algorithmic}[1]
    \REQUIRE $S = \{t \in [T] \mid |Z_t| = s_t\}$ the set of indices $t$ such that $Z_t$ is tight, $g_t = (\nabla u_t(\bw^{(Z_t)}))^2_{a^*}$, $c_t = \min_{a \in Z_t} (\bw^{(Z_t)})^2_{a}$ for each $t \in [T]$, and $\theta = s' - \sum_{t = 1}^T |Z_t|$.
    \ENSURE $A, B \subseteq [T]$ maximizing $\sum_{t \in A} g_t - \sum_{t \in B} c_t$ subject to $A \cap S \subseteq B$ and $|A| \le |B| + \theta$.
	\STATE Initialize $A_0 \gets \emptyset$ and $B_0 \gets \emptyset$.
    \STATE Let $S = \{t \in [T] \mid |Z_t| = s_t\}$.
    \STATE Sort $t \in [T] \setminus S$ according to $g_t$ into the priority queue $Q_1$ in descending order.
    \STATE Sort $t \in [T]$ according to $c_t$ into the priority queue $Q_2$ in ascending order.
    \STATE Sort $t \in S$ according to $g_t - c_t$ into the priority queue $Q_3$ in descending order.
    \FOR{$i = 1,\cdots,T$}
        \STATE Let $\alpha$, $\beta$ and $\gamma$ be the top elements in $Q_1$, $Q_2$, and $Q_3$, respectively.
        \IF{$g_\alpha - c_\beta \bfone \{|A_{i-1}| = |B_{i-1}| + \theta\} \le 0$ and $g_\gamma - c_\gamma \le 0$}
            \STATE \textbf{return} $A_{i-1}$ and $B_{i-1}$
        \ELSE
            \IF{$g_\alpha - c_\beta \bfone \{|A_{i-1}| = |B_{i-1}| + \theta\} \ge g_\gamma - c_\gamma$}
                \STATE $A_i \gets A_{i-1} + \alpha$ and remove $\alpha$ from $Q_1$.
                \IF{$|A_{i-1}| = |B_{i-1}| + \theta$}
                \STATE $B_i \gets B_{i-1} + \beta$ and remove $\beta$ from $Q_2$.
                \IF{$\beta \in S$}
                    \STATE Remove $\beta$ from $Q_3$ and add $\beta$ to $Q_1$.
                \ENDIF
                \ENDIF
            \ELSE
                \STATE $A_i \gets A_{i-1} + \gamma$ and $B_i \gets B_{i-1} + \gamma$.
                \STATE Remove $\gamma$ from $Q_3$.
            \ENDIF
        \ENDIF
    \ENDFOR
    \STATE \textbf{return} $A_T$ and $B_T$
\end{algorithmic}
\end{algorithm}

Here we give a fast greedy method for calculating the gain of each atom. This algorithm can be executed in $\mathrm{O}(T \log T)$ time. The detailed description of this algorithm is given in \Cref{alg:greedy-selection-average-sparsity}.

\begin{proposition}
    \Cref{alg:greedy-selection-average-sparsity} returns an optimal solution in $\mathrm{O}(T \log T)$ time.
\end{proposition}

\begin{proof}
    First we show the validity of the algorithm.

    Before proving the optimality of the output, we note that the marginal gain of each step of the algorithm is largest among all the feasible updates. Let us consider the addition of $\alpha$ to $A_{i-1}$. There are three cases of updates. If $\alpha \in S \setminus B_{i-1}$ is added to $A_{i-1}$, we must also add $\alpha$ to $B_{i-1}$. If $\alpha \not\in S \setminus B_{i-1}$ and $|A_{i-1}| = |B_{i-1}| + \theta$, adding $\beta \not\in B_{i-1}$ with smallest cost $c_t$ is the best choice. If $\alpha \not\in S \setminus B_{i-1}$ and $|A_{i-1}| < |B_{i-1}| + \theta$, not changing $B_{i-1}$ is the best choice. \Cref{alg:greedy-selection-average-sparsity} selects the best one from these cases.
    
    We show $(A_i, B_i)$ be optimal among feasible solutions such that $|A| = i$ by induction on $i$. It is clear that $(A_0, B_0)$ is optimal among feasible solutions such that $|A| = 0$.
    
    Now we assume $(A_{i-1}, B_{i-1})$ is optimal among feasible solutions such that $|A| = i-1$. Let $(A'_i, B'_i)$ be an optimal solution among feasible solutions such that $|A| = i$. If there exist $\alpha \in A'_i \setminus A_{i-1}$ and $\beta \in B'_i \setminus B_{i-1}$ such that $(A_{i-1} + \alpha, B_{i-1} + \beta)$ and $(A'_i - \alpha, B'_i - \beta)$ are both feasible, we obtain
    \begin{align*}
        \sum_{t \in A_i} g_t - \sum_{t \in B_i} c_t
        &\ge \left(\sum_{t \in A_{i-1}} g_t - \sum_{t \in B_{i-1}} c_t \right) + \left( g_{\alpha_i} - c_{\beta_i} \right)\\
        &\ge \left(\sum_{t \in A'_{i-1}} g_t - \sum_{t \in B'_{i-1}} c_t \right) + \left(g_{\alpha'} - c_{\beta'} \right)\\
        &\ge \left(\sum_{t \in A'_{i}} g_t - \sum_{t \in B'_{i}} c_t \right),
    \end{align*}
    which proves the optimality of $(A_i,B_i)$. The second inequality is because the marginal gain of $\alpha_i$ (or possibly $\alpha_i$ and $\beta_i$) is largest among feasible additions. In the same way, if there exist $\alpha \in A'_i \setminus A_{i-1}$ such that $(A_{i-1} + \alpha, B_{i-1})$ and $(A'_i - \alpha, B'_i)$ are both feasible, then $(A_i,B_i)$ is optimal.

    We show the existence of such an $\alpha$ or pair $(\alpha, \beta)$. Since $|A'_i| > |A_{i-1}|$, we have $A'_i \setminus A_{i-1} \neq \emptyset$. Let $\alpha \in A'_i \setminus A_{i-1}$ be an arbitrary element. If $\alpha \in B'_i \setminus B_{i-1}$, the pair $(\alpha, \alpha)$ satisfies the condition. If $\alpha \not\in B'_i \setminus B_{i-1}$ and $|A_{i-1}| < |B_{i-1}| + \theta$, then $\alpha$ satisfies the condition. If $\alpha \not\in B'_i \setminus B_{i-1}$ and $|A_{i-1}| = |B_{i-1}| + \theta$, we have $|B'_i| \ge |A'_i| - \theta > |B_{i-1}|$, then $B'_i \setminus B_{i-1} \neq \emptyset$. Therefore a pair of $\alpha$ and an arbitrary $\beta \in B'_i \setminus B_{i-1}$ satisfies the condition.

    Finally we consider the running time of this algorithm. Sorting requires $\mathrm{O}(T \log T)$ time. Each iteration requires $\mathrm{O}(\log T)$ time. Thus, the total running time is $\mathrm{O}(T \log T)$.
\end{proof}

\subsection{Proof of \Cref{thm:replacement-omp-runtime}}
\begin{proof}
    In each iteration, we need to find an atom with the largest gain and the corresponding new supports $(Z_1', \dots, Z_t')$.
    This can be done in $\mathrm{O}(nT\log T)$ time.
    Furthermore, we need to compute a new coefficient $\bw^{(Z'_t)}_t = \bA_{{Z'_t}}^+ \by_t$
    for the new support ${Z'_t}$ ($t \in [T]$), 
    where $\bA^+$ is the pseudo inverse.
    This can be done efficiently via maintaining the QR-decomposition of $\bA_{Z_t}$ under rank-two update~\citep{Golub2012matrix} with a cost of $\mathrm{O}(s^2 + ds) = \mathrm{O}(ds)$ time for each matrix.
    Thus each iteration requires $\mathrm{O}(T(n\log T + ds))$ time, which proves the theorem.
\end{proof}

\section{Online dictionary selection}\label{sec:online}

Online dictionary selection is the problem of selecting a dictionary at each round. 
At each round $t$, the player selects a dictionary $X_t \subseteq V$ with $|X_t| \le k$, then the adversary reveals a data point $\by_t \in \R^d$. 
Then the player gains with respect to the best $s$-sparse approximation to $\by_t$ with the selected dictionary $X_t$:
\begin{align}
\max_{\bw \colon \norm{\bw}_0 \le s} u(\by_t, \bA_{X_t} \bw),
\end{align}
where $\bA_{X_t}$ is the matrix obtained by arranging all vectors contained in $X_t$.
Let $g_t(X) = \max_{Z \subseteq X \colon |Z| \le s} f_t(Z)$ be the objective function at the $t$th round, where $f_t(Z) = \max_{\bw \colon \norm{\bw}_0 \leq s} u(\by_t, A_{Z} \bw)$.
In the following, we provide the online versions of algorithms for offline dictionary selection: Online \ModularApproximation{}, Online \ReplacementGreedy{}, and Online \ReplacementOMP{}.

\subsection{Online \ModularApproximation{}}\label{subsec:modular-approx}
The first algorithm is based on \ModularApproximation{} for offline dictionary selection, which was proposed by \citet{Krause2010} and given an improved analysis by \citet{Das2011}.
At each round $t$, we consider a function $\tilde{f}_t(Z) = \sum_{a \in Z} f_t(a|\emptyset)$, which is a modular approximation of $f_t$.
Intuitively, the modular approximation $\tilde{f}_t$ ignores the interactions among the atoms.
We define the surrogate objective $\tilde{g}_t$ as 
\begin{align}
\tilde{g}_t(X) = \max_{Z \subseteq X \colon |Z| \le s} \tilde{f}_t(Z).
\end{align} 
It is easy to show that $\tilde{g}_t$ is monotone submodular.
Hence, we can apply the online greedy algorithm~\citep{Streeter2009} to these surrogate functions.



Let $u_t(\bw) \coloneqq u(\by_t, \bA\bw)$ for $t \in [T]$.
Assuming the strong concavity and smoothness of $u_t$, the original objective function $g_t$ can be bounded from lower and upper with the surrogate function $\tilde{g}_t$. A similar result is given in \citet{Elenberg2016} for offline sparse regression.

\begin{lemma}\label{lem:surrogate}
	Suppose $u_t$ is $m_1$-strongly concave and $M_1$-smooth on $\Omega_1$, and $m_s$-strongly concave and $M_s$-smooth on $\Omega_s$.
    Then,
   \[ 
    \frac{m_1}{M_s} \tilde{g}_t(X) \le g_t(X) \le \frac{M_1}{m_s} \tilde{g}_t(X).
   \]
\end{lemma}

\begin{proof}
	Let $Z \subseteq V$ be an arbitrary subset such that $|Z| \le s$. Since the submodularity ratio $\gamma_{\emptyset, s}$ of $f$ is no less than $m_s / M_1$~\citep{Elenberg2016}, 	\begin{equation*}
		\frac{m_s}{M_1} f_t(Z) \le \sum_{a \in Z} \tilde{f}_t(a) = \tilde{f}_t(Z).
	\end{equation*}
	As this bound holds for any $Z \subseteq V$ of size no more than $s$, we have
	\begin{equation*}
		g_t(X) = \max_{Z \subseteq X \colon |Z| \le s} f_t(Z) \le \frac{M_1}{m_s} \max_{Z \subseteq X \colon |Z| \le s} \tilde{f}_t(Z) = \frac{M_1}{m_s} \tilde{g}_t(X).
	\end{equation*}
	
	Next we prove the lower bound of $g_t(X)$. From the optimality of $\bw^{(Z)}$, for any $\bz$ such that $\supp(\bz) \subseteq Z$,
    \begin{align*}
		f_t(Z)
		&= u_t(\bw^{(Z)}) - u_t(\bfzero)\\
		&\ge u_t(\bz) - u_t(\bfzero)\\
		&\ge \langle \nabla u_t (\bfzero), \bz \rangle - \frac{M_s}{2} \| \bz \|^2
    \end{align*}
	where the last inequality is due to the strong concavity of $u_t$. Using $\bz = \frac{1}{M_s} (\nabla u_t(\bfzero) )_Z$, we obtain
    \begin{equation}\label{eq:surrogate-bound-concave}
		f_t(Z)
		\ge \frac{1}{2M_s} \| (\nabla u_t (\bfzero) )_Z \|^2.
    \end{equation}
	On the other hand, from the smoothness of $u_t$, we have for all $a \in Z$,
    \begin{align*}
		f_t(a)
		&= u_t(\bw^{(a)}) - u_t(\bfzero)\\
		&\le \langle \nabla u_t (\bfzero), \bw^{(a)} \rangle - \frac{m_1}{2} \| \bw^{(a)} \|^2\\
		&\le \max_{c \in \R} \langle \nabla u_t (\bfzero), c \bchi_a \rangle - \frac{m_1}{2} \| c \bchi_a \|^2\\
		&= \frac{1}{2m_1} (\nabla u_t(\bfzero))_a^2.
    \end{align*}
	Summing up for all $a \in Z$, we obtain
    \begin{equation}\label{eq:surrogate-bound-smooth}
		\tilde{f}_t(Z) = \sum_{a \in Z} f_t(a) \le \frac{1}{2m_1} \| (\nabla u_t(\bfzero))_Z \|^2.
	\end{equation}
	Combining \eqref{eq:surrogate-bound-concave} and \eqref{eq:surrogate-bound-smooth}, we obtain the lower bound
    \begin{equation*}
		f_t(Z) \ge \frac{m_1}{M_s} \tilde{f}_t(Z),
	\end{equation*}
	which proves the lower bound of $g_t(X)$ in the same way as the upper bound.
\end{proof}

The expected regret of this algorithm can be bounded as follows.
\begin{theorem}\label{thm:modular-approx}
	Let $\alpha = (1 - \frac{1}{\rme}) \frac{m_1 m_s}{M_1 M_s}$. 
    The expected $\alpha$-regret of the modular approximation algorithm after $T$ rounds is bounded as follows:
    \[
	\regret_\alpha(T) \le \frac{k \Delta_{\max} m_1}{M_s} \sqrt{2T \ln n}
    \]
    where $n = |V|$ and $\Delta_{\max} = \max_{a \in V} \max_{t \in [T]} f_t(a|\emptyset)$.
\end{theorem}
\begin{proof}
	Applying the regret bound for online submodular maximization \cite{Streeter2009}, we obtain
	\begin{align}
		\left( 1 - \frac{1}{\rme} \right) \sum_{t=1}^T \tilde{g}_t(X^*) - \sum_{t=1}^T \tilde{g}_t(X_t) \le k \Delta_{\max} \sqrt{2 T \ln n}.
	\end{align}
    since the gains for each subroutine are bounded by $\Delta_{\max}$.
	From Lemma~\ref{lem:surrogate}, we obtain the bound in the statement.
\end{proof}

For the squared $\ell^2$-utility function $u(\by,\bx) = \frac{1}{2}\norm{\by}_2^2 - \frac{1}{2}\norm{\by-\bx}_2^2$, $\alpha$ is equal to an approximation ratio shown in \citet{Das2011}.
\begin{corollary}
	For the squared $\ell^2$-utility function, the expected regret of the modular approximation algorithm is
    \[
		\regret_\alpha(T) \le \frac{k \Delta_{\max} }{\sigmamax^2(\bA, s)} \sqrt{2T \ln n},
    \]
    where $\alpha = (1 - \frac{1}{\rme})\frac{\sigmamin^2(\bA, s)}{\sigmamax^2(\bA, s)}$.
\end{corollary}

\subsection{Online \ReplacementGreedy{}}\label{app:online-replacement-greedy}
In the following, we provide online adoptation of \ReplacementGreedy{}.
Similarly to \citet{Streeter2009}, we use $k$ expert algorithms $\caA^1, \cdots, \caA^k$ as subroutines. 
At each round, online \ReplacementGreedy{} selects a set of $k$ elements $a_t^1, \cdots, a_t^k$ according to the expert algorithms $\caA^1,\cdots,\caA^k$, respectively. 
After the target point $\by_t$ is revealed, the algorithm decides the feedback to the subroutines by considering how $Z_t$ changes if $a_t^1,\cdots,a_t^k$ are added to $X$ sequentially. 
As in the offline version of \ReplacementGreedy{}, we start with $Z_t^0 = \emptyset$ and consider adding $a_t^i$ to $Z_t$ or not with keeping $|Z_t| \le s$ for each $i = 1, \cdots, k$.
Denoting $Z_t$ at the $i$th step by $Z_t^i$, we can write the feedback given to the subroutine $\caA^i$ as $\Delta_t(\cdot, Z_t^{i-1})$ where
\begin{equation*}
    \Delta_t(a, Z_t^{i}) =
    \begin{cases}
        \displaystyle f_t(Z_t^{i} + a) - f_t(Z_t^{i}) & (i < s)\\
        \displaystyle \max \left\{ 0, \max_{a' \in Z_t^{i}} \left\{ f_t(Z_t^{i} - a' + a) - f_t(Z_t^{i}) \right\} \right\} & (i \ge s)
    \end{cases}
\end{equation*}
is the gain obtained by adding $a$ to $Z_t^i$. If $\Delta_t(a^i_t, Z_t^{i-1}) > 0$, the algorithm updates $Z_t$ by adding $a_t^i$ and, if $i > s$, removing $a'$ that maximizing $f_t(Z_t^{i-1} - a' + a_t^i)$. For each $a \in V$, the value of gain $\Delta_t(a, Z_t^{i-1})$ is given to $\caA^i$ as the feedback about $a$. A pseudocode description of our algorithm is shown in \Cref{alg:online-replacement}.

\begin{theorem}\label{thm:online-replacement-greedy}
Assume that $u_t$ is $m_{2s}$-strongly concave on $\Omega_{2s}$ and $M_{s,2}$-smooth on $\Omega_{s,2}$ for $t \in [T]$. 
Then the online replacement greedy algorithm achieves the regret bound
$\regret_{\alpha}(T) \le \sum_{i=1}^k r_{i}$, where $r_i$ is the regret of the online greedy selection subroutine $\caA^i$ for $i \in [k]$ and
\[
\alpha = \left( \frac{m_{2s}}{M_{s,2}} \right)^2 \left(1 - \exp\left(- \frac{M_{s,2}}{m_{2s}} \right)\right).
\]
In particular, if we use the hedge algorithm as the online greedy selection subroutines, we obtain $\regret_\alpha(T) \leq k\sqrt{2T\ln{n}}$.
\end{theorem}

\begin{corollary}
For the squared $\ell^2$-utility function, 
\[
\alpha \ge \left( \frac{\sigmamin^2(\bA, 2s)}{\sigmamax^2(\bA, 2)} \right)^2 \left(1 - \exp\left(- \frac{\sigmamax^2(\bA, 2)}{\sigmamin^2(\bA, 2s)} \right)\right).
\]
\end{corollary}
\begin{proof}[Proof of \Cref{thm:online-replacement-greedy}]
    We provide a lower bound on the sum of the $i$th step marginal gains of the algorithm. Let $Z_t^*$ be an optimal sparse subset of $X^*$ for $f_t$, i.e., $Z_t^* \in \argmax_{Z \subseteq X^* \colon |Z| \le s} f_t(Z)$. Then we have
    \begin{align}
        \sum_{t = 1}^T \Delta_t (a_t^i, Z_{t}^{i-1})
        &\ge \max_{x \in V} \sum_{t = 1}^T \Delta_t (x, Z_{t}^{i-1}) - r_i \nonumber \\
        &\ge \frac{1}{k} \sum_{a \in X^*} \sum_{t = 1}^T \Delta_t (a, Z_{t}^{i-1}) - r_i \nonumber \\
        &\ge \frac{1}{k} \sum_{t = 1}^T \sum_{a \in Z_t^*} \Delta_t (a, Z_{t}^{i-1}) - r_i \nonumber \\
        &\ge \frac{1}{k} \sum_{t = 1}^T \left(C_1 f_t(Z_i^*) - C_2 f_t(Z_{t}^{i-1}) \right) - r_i \label{eq:online-replacement-mwu-bound}
    \end{align}
    where $C_1 = \frac{m_{2s}}{M_{s,2}}$ and $C_2 = \frac{M_{s,2}}{m_{2s}}$. The first inequality is due to the regret bound for the subroutine $\caA^i$. The last inequality is due to Lemma \ref{lem:replacement-greedy-marginal}.
    Now the theorem directly follows from \Cref{lem:misc}.
\end{proof}

\subsection{Online \ReplacementOMP{}}
In this section, we consider an online version of \ReplacementOMP{}. This algorithm is the same as Online \ReplacementGreedy{} except the gain at each step. The gain obtained when $a$ is added to $Z_t^{i}$ is 
\begin{equation*}
        \displaystyle \frac{1}{2M_{s,2}} \left( \nabla u_t(\bw_t^{(Z^i_t)})\right)^2_a
\end{equation*}
when $i < s$, and 
\begin{equation*}
        \displaystyle \max \left\{0, \frac{1}{2M_{s,2}} \left( \nabla u_t(\bw_t^{(Z^i_t)})\right)^2_a - \min_{a' \in Z^i_t} \frac{M_{s,2}}{2} \left( \bw_t^{(Z^i_t)} \right)^2_{a'} \right\}
\end{equation*}
when $i \ge s$, where $\bw_t^{(Z^i_t)} \in \argmax_{\bw: \supp(\bw) \subseteq Z^i_t} u_t(\bw)$.

\begin{theorem}\label{thm:online-replacement-omp}
Assume that $u_t$ is $m_{2s}$-strongly concave on $\Omega_{2s}$ and $M_{s,2}$-smooth on $\Omega_{s,2}$ for $t \in [T]$. 
    Then Online \ReplacementOMP{} algorithm achieves the regret bound
$\regret_{\alpha}(T) \le \sum_{i=1}^k r_{i}$, where $r_i$ is the regret of the online greedy selection subroutine $\caA^i$ for $i \in [k]$ and
\[
\alpha = \left( \frac{m_{2s}}{M_{s,2}} \right)^2 \left(1 - \exp\left(- \frac{M_{s,2}}{m_{2s}} \right)\right).
\]
In particular, if we use the hedge algorithm as the online greedy selection subroutines, we obtain $\regret_\alpha(T) \leq k\sqrt{2T\ln{n}}$.
\end{theorem}

\begin{proof}
    Since $f_t$ is $M_{s,2}$-smooth on $\Omega_{s,2}$, it holds that for any $a, a' \in V$ and $Z_t \subseteq V$ of size at most $s$,
    \begin{equation*}
        f_t(Z_t - a' + a) - f_t(Z_t) \ge \frac{1}{2M_{s,2}} \left( \nabla u_t(\bw_t^{(Z_t)}) \right)_a^2 - \frac{M_{s,2}}{2} \left( \bw_t^{(Z_t)} \right)_{a'}^2.
    \end{equation*}
    In addition, we have
    \begin{equation*}
        \frac{1}{2M_{s,2}} \| ( \nabla u_t (\bw^{(Z)}) )_{X \setminus Z} \|^2 - \frac{M_{s,2}}{2} \| (\bw^{(Z)})_{Z \setminus X} \|^2 \le \frac{m_{2s}}{M_{s,2}} f(X) - \frac{M_{s,2}}{m_{s2}} f(Z)
    \end{equation*}
    from the proof of \Cref{lem:replacement-greedy-marginal}.

    We provide a lower bound on the $i$th step marginal gain of the algorithm. Let $Z_t^*$ be an optimal sparse subset of $X^*$ for $f_t$, i.e., $Z_t^* \in \argmax_{Z \subseteq X^* \colon |Z| \le s} f_t(Z)$. If $i \le s$, then $|Z_t^{i-1}| < s$ holds for all $t$. Then we have
    \begin{align}
        \sum_{t = 1}^T \Delta_t (a_t^i, Z_{t}^{i-1})
        &= \sum_{t = 1}^T \left\{ f_t(Z_t^{i-1} + a_t^i) - f_t( Z_{t}^{i-1}) \right\}\nonumber \\
        &\ge \sum_{t = 1}^T \frac{1}{2M_{s,2}} \left( \nabla u_t (\bw_t^{(Z_t)}) \right)^2_{a_t^i} \nonumber \\
        &\ge \max_{a^i \in V} \sum_{t = 1}^T \frac{1}{2M_{s,2}} \left( \nabla u_t (\bw_t^{(Z_t)}) \right)^2_{a^i} - r_i \nonumber \\
        &\ge \frac{1}{k} \sum_{a \in X^*} \sum_{t = 1}^T \frac{1}{2M_{s,2}} \left( \nabla u_t (\bw_t^{(Z_t)}) \right)^2_{a} - r_i \nonumber \\
        &\ge \frac{1}{k} \sum_{t = 1}^T \sum_{a \in Z_t^* \setminus Z_t} \frac{1}{2M_{s,2}} \left( \nabla u_t (\bw_t^{(Z_t)}) \right)^2_{a} - r_i \nonumber \\
        &\ge \frac{1}{k} \sum_{t = 1}^T \frac{1}{2M_{s,2}} \left\| \nabla u_t (\bw_t^{(Z_t)}) \right\|^2_{Z^*_t \setminus Z_t} - r_i \nonumber \\
        &\ge \frac{1}{k} \sum_{t = 1}^T \left(\frac{m_{s,2}}{M_{s,2}} f_t(Z_i^*) - \frac{M_{s,2}}{m_{2s}} f_t(Z_{t}^{i-1}) \right) - r_i. \nonumber 
    \end{align}
    Otherwise, $|Z_t^{i-1}| = s$ holds for all $t$, therefore
    \begin{align}
        \sum_{t = 1}^T \Delta_t (a_t^i, Z_{t}^{i-1})
        &\ge \sum_{t = 1}^T \max \left\{ 0, \max_{a'_t \in Z_t} \left\{ f_t(Z_t - a'_t + a_t^i) - f_t(Z_t) \right\} \right\} \nonumber \\
        &\ge \sum_{t = 1}^T \max \left\{ 0, \frac{1}{2M_{s,2}} \left( \nabla u_t (\bw_t^{(Z_t)}) \right)^2_{a_t^i} - \min_{a'_t \in Z_t} \frac{M_{s,2}}{2} \left( \bw_t^{(Z_t)} \right)^2_{a'_t} \right\} \nonumber \\
        &\ge \max_{a^i \in V} \sum_{t = 1}^T \max \left\{ 0, \frac{1}{2M_{s,2}} \left( \nabla u_t (\bw_t^{(Z_t)}) \right)^2_{a^i} - \min_{a'_t \in Z_t} \frac{M_{s,2}}{2} \left( \bw_t^{(Z_t)} \right)^2_{a'_t} \right\} - r_i \nonumber \\
        &\ge \frac{1}{k} \sum_{a \in X^*} \sum_{t = 1}^T \max \left\{ 0, \frac{1}{2M_{s,2}} \left( \nabla u_t (\bw_t^{(Z_t)}) \right)^2_a- \min_{a'_t \in Z_t} \frac{M_{s,2}}{2} \left( \bw_t^{(Z_t)} \right)^2_{a'_t} \right\} - r_i \nonumber\\
        &\ge \frac{1}{k} \sum_{t = 1}^T \sum_{a \in Z_t^* \setminus Z_t} \left\{ \frac{1}{2M_{s,2}} \left( \nabla u_t (\bw_t^{(Z_t)}) \right)^2_a- \frac{M_{s,2}}{2} \left( \bw_t^{(Z_t)} \right)^2_{\pi_t(a)} \right\} - r_i \nonumber \\
        &\ge \frac{1}{k} \sum_{t = 1}^T \left\{ \frac{1}{2M_{s,2}} \left\| ( \nabla u_t (\bw_t^{(Z_t)}) )_{Z^*_t \setminus Z_t} \right\|^2 - \frac{M_{s,2}}{2} \left\| (\bw_t^{(Z_t)})_{Z_t \setminus Z_t^*} \right\|^2 \right\} - r_i \nonumber \\
        &\ge \frac{1}{k} \sum_{t = 1}^T \left(\frac{m_{2s}}{M_{s,2}} f_t(Z_i^*) - \frac{M_{s,2}}{m_{2s}} f_t(Z_{t}^{i-1}) \right) - r_i. \label{eq:replacement-omp-mwu-bound}
    \end{align}
    where a map $\pi_t \colon Z^*_t \setminus Z_t \to Z_t \setminus Z_t^*$ is an arbitrary bijection for each $t$.

    Combining with \Cref{lem:misc}, we obtain the theorem.
\end{proof}

\begin{algorithm}
    \caption{Online \ReplacementGreedy{} \& Online \ReplacementOMP}\label{alg:online-replacement}
\begin{algorithmic}[1]
	\STATE Initialize online greedy selection subroutines $\caA^i$ for $i = 1, \dots, k$.
    \FOR{$t= 1, \dots, T$}
        \STATE Initialize $X_t^0 \gets \emptyset$ and $Z_t^0 \gets \emptyset$ for all $t \in [T]$.
    \FOR{$i = 1, \dots, k$}
    	\STATE Pick $a_t^i \in V$ according to $\caA^i$.
        \STATE Set $X_t^i \gets X_t^{i-1} + a_t^i$.
    \ENDFOR
   	\STATE Play $X^k_t$ and observe $\by_t$.
    \FOR{$i = 1, \dots, k$}
        \STATE To the subroutine $\caA_i$, feed the gain of $a$ defined as
            \begin{itemize}
                \item $\Delta_t(a, Z_{t}^{i-1})$ \hfill (Online \ReplacementGreedy) \\
                \item $
                \begin{cases}
                    \displaystyle \frac{1}{M_{s,2}} \left(\nabla u_t\left(\bw_t^{(Z_t)}\right)\right)^2_a & \text{ if $i \leq s$, } \\
                    \displaystyle \max \left\{0, \frac{1}{M_{s,2}} \left( \nabla u_t(\bw_t^{(Z^{i-1}_t)})\right)^2_a - M_{s,2} \min_{a'_t \in Z^{i-1}_t} \left( \bw_t^{(Z^{i-1}_t)} \right)^2_{a_t'} \right\} & \text{otherwise}
                \end{cases}
                $
                \\\hfill(Online \ReplacementOMP)
            \end{itemize}
        \STATE Do the optimal replacement of $Z_t^{i-1}$ with respect to $a_t^i$ that achieves the above gain for \ReplacementGreedy{} or \ReplacementOMP{}, and obtain $Z_t^i$.
    \ENDFOR
    \ENDFOR
\end{algorithmic}
\end{algorithm}

\section{Experiments on dimensionality reduced data}\label{sec:further}
In this section, we conduct experiments on the task called \emph{image restoration}.
In this task, we are given an incomplete image, that is, a portion of its pixels are missing. 
First, we divide this incomplete image into small patches of $8 \times 8$ pixels.
Then we regard each of these patches as a data point $\by_t$, and aim to select a dictionary that yields a sparse representation of these patches.
In the procedure of the algorithms, the loss is evaluated only on the given pixels.
Finally, we restore the original image by replacing each patch with a sparse approximation using the selected dictionaries, and the loss is evaluated on the whole pixels.

First we conduct experiments with synthetic datasets to investigate the behavior of the algorithms. For each of the training and test datasets, we generate a bit mask such that each value takes $0$ or $1$ with equal probability. We give the masked training dataset to the algorithms and let them learn a dictionary. With this dictionary, we create the sparse representation of each data point in the test dataset with only unmasked elements and evaluate its residual variance with the whole elements. \Cref{fig:offline_synthetic_inpainting_T100_time} and \ref{fig:offline_synthetic_inpainting_T100_error} are the results for smaller datasets of $T = 100$, and \Cref{fig:offline_synthetic_inpainting_T1000_time} and \ref{fig:offline_synthetic_inpainting_T1000_error} are the results for larger datasets of $T = 1000$. In both experiments, we can see the relationship of the algorithms' performance is similar to the one in the non-masked settings, \Cref{fig:offline_synthetic_T100_time}, \ref{fig:offline_synthetic_T100_error}, \ref{fig:offline_synthetic_T1000_time}, and \ref{fig:offline_synthetic_T1000_error}.

\begin{figure*}
\centering
\subfigure[$T = 100$, time]{
	\includegraphics[width=0.4\textwidth]{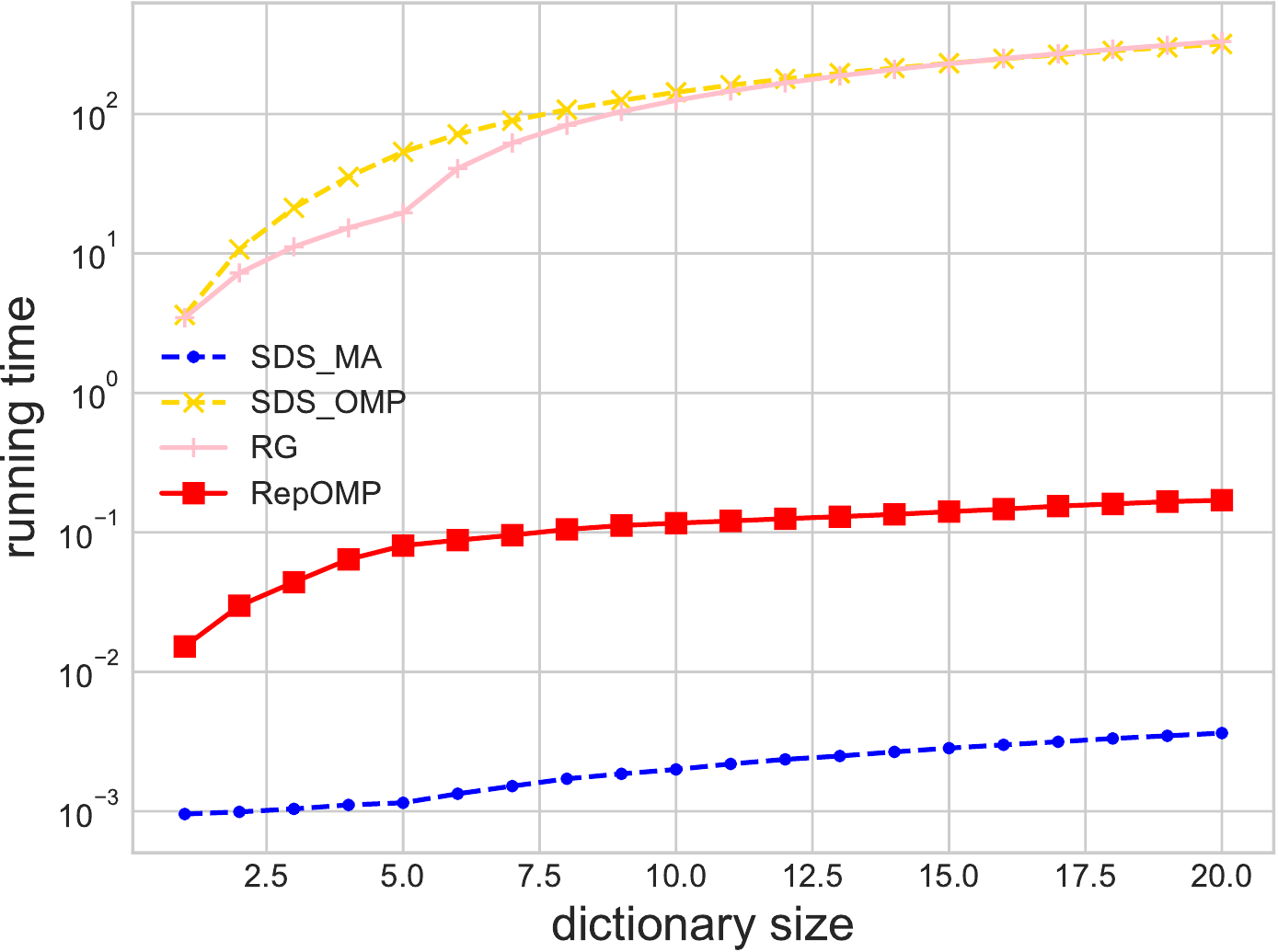}\label{fig:offline_synthetic_inpainting_T100_time}
}
\subfigure[$T = 100$, residual]{
	\includegraphics[width=0.4\textwidth]{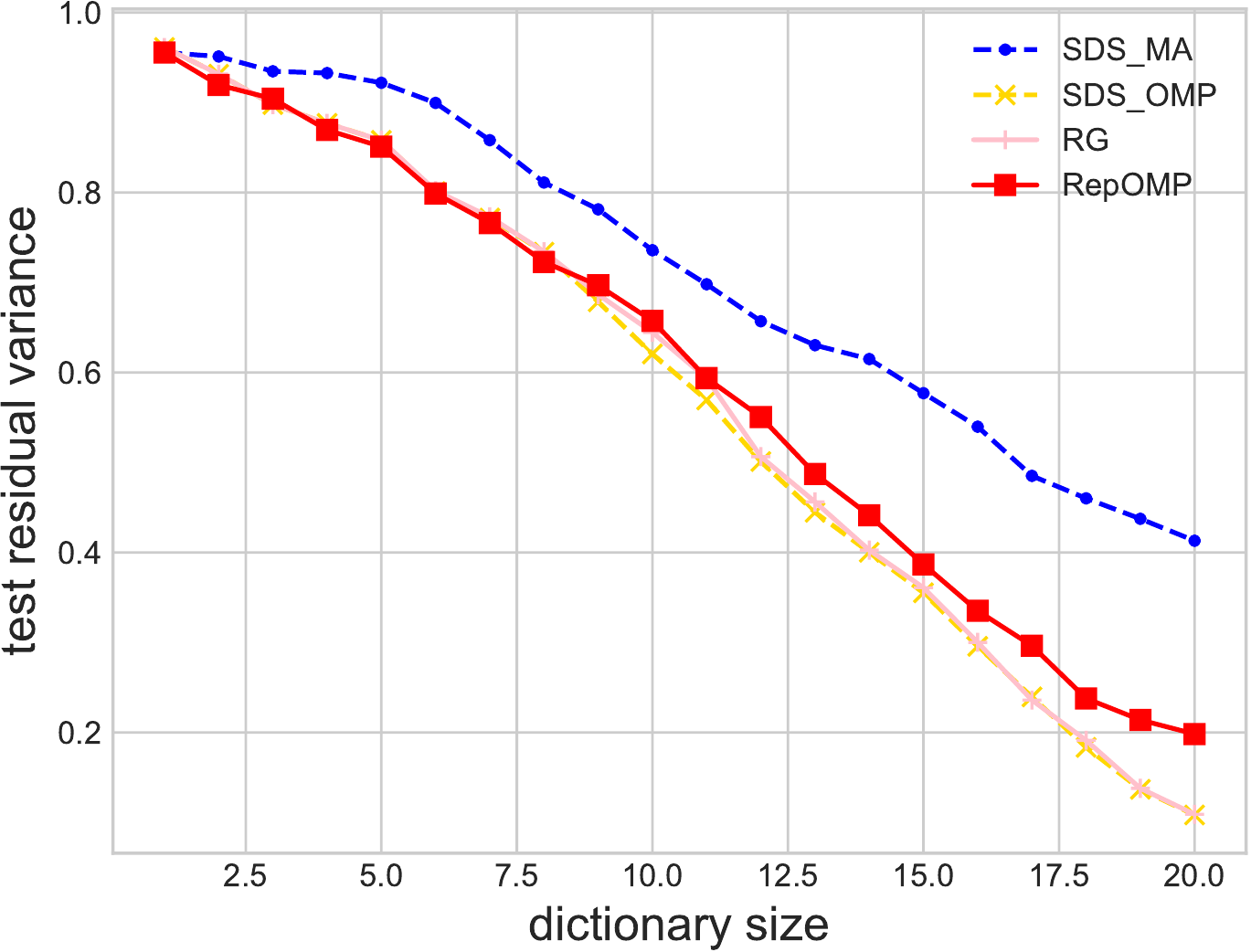}\label{fig:offline_synthetic_inpainting_T100_error}
}

\subfigure[$T = 1000$, time]{
	\includegraphics[width=0.4\textwidth]{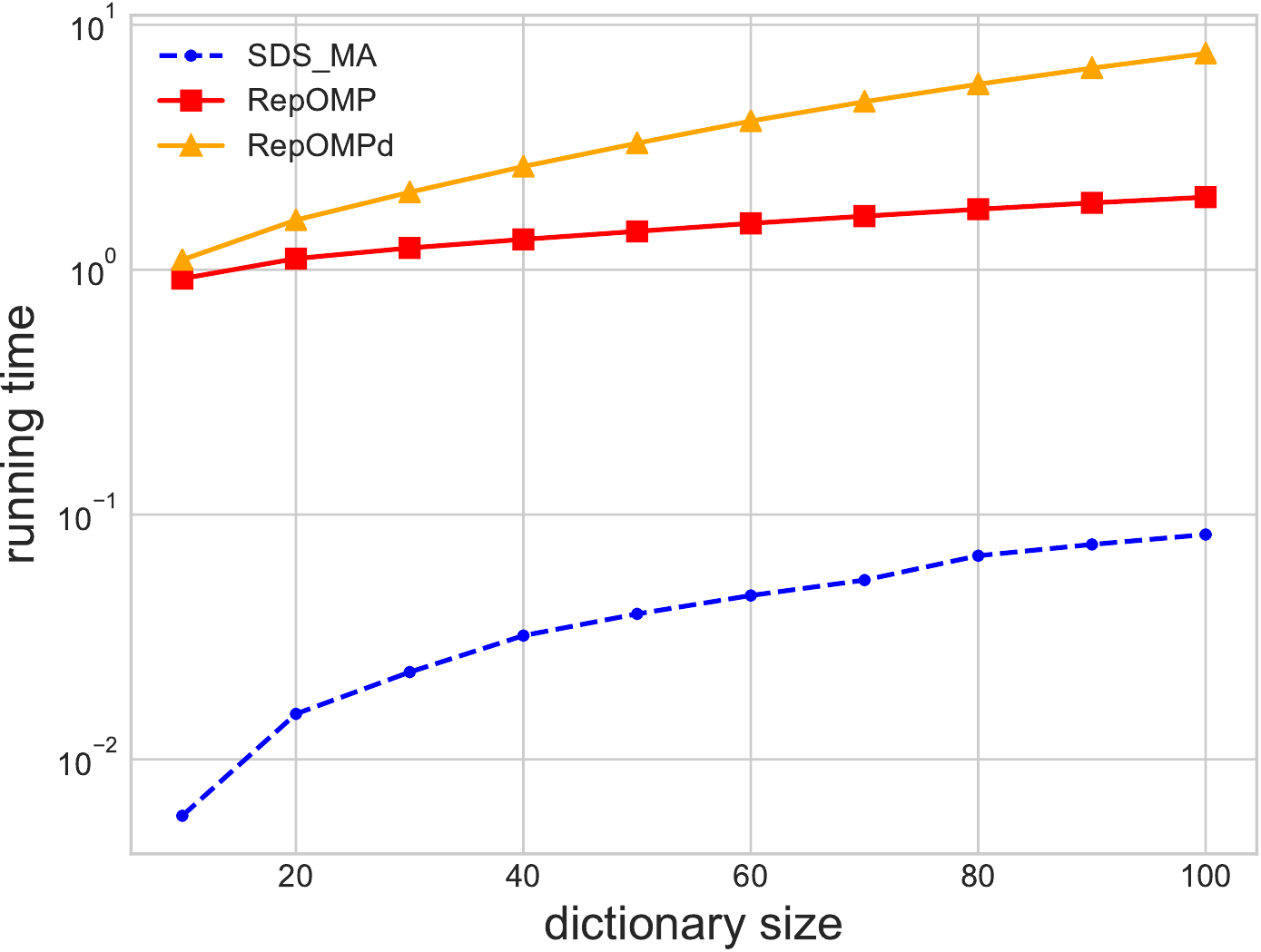}\label{fig:offline_synthetic_inpainting_T1000_time}
}
\subfigure[$T = 1000$, residual]{
	\includegraphics[width=0.4\textwidth]{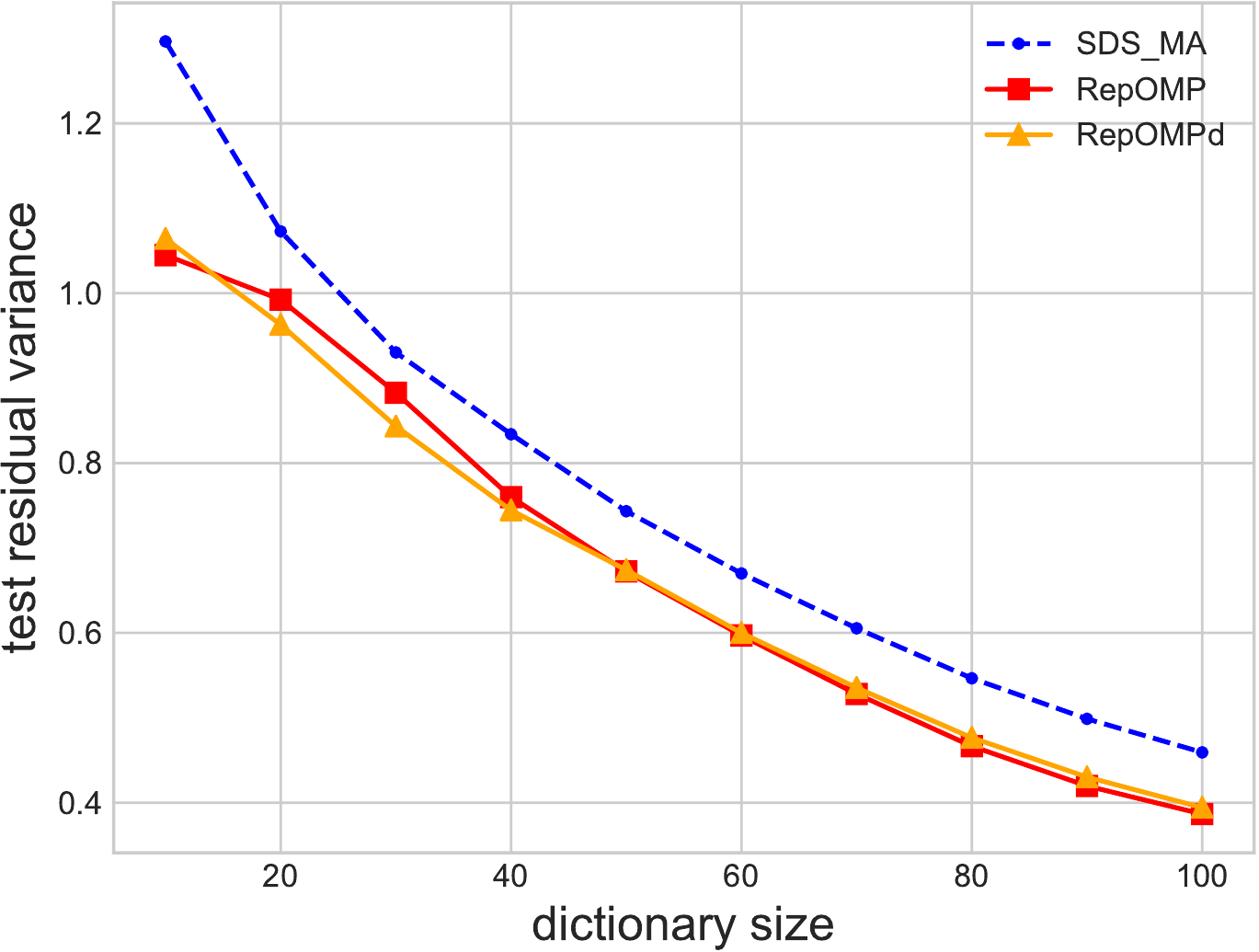}\label{fig:offline_synthetic_inpainting_T1000_error}
}
\caption{The experimental results for dimensionality reduced synthetic datasets. In all figures, the horizontal axis indicates the size of the output dictionary. 
    \subref{fig:offline_synthetic_inpainting_T100_time} and \subref{fig:offline_synthetic_inpainting_T100_error} are the results for $T = 100$. \subref{fig:offline_synthetic_inpainting_T1000_time} and \subref{fig:offline_synthetic_inpainting_T1000_error} are the results for $T = 1000$. For each setting, we give the plot of the running time and the test residual variance.}
\end{figure*}

In order to illustrate the advantage of the average sparsity to ordinary dictionary learning (the individual sparsity), we give image restoration examples with real-world images.
We use \ReplacementOMP{} for both the individual sparsity and the average sparsity.
With setting $s_t = s$ for all $t \in [T]$, the parameters $k$, $s$, and $s'$ are determined with the grid search.
We apply \ReplacementOMP{} to incomplete images and obtain a dictionary.
Then with this dictionary, we repeatedly compute the sparse representation of patches in the input image while shifting a single pixel. OMP is used for obtaining the sparse representation. When calculating the coefficients of the sparse representation of each patch, we use only the observed pixels and restore the whole pixels with these coefficients.
We take the median value of all the restored patches for each pixel.
In \Cref{fig:image}, the input image, the image restored with the individual sparsity, and the image restored with the individual sparsity are shown with PSNR ratios.
The method with the average sparsity obtains higher PSNR ratios than one with the individual sparsity for all the images.
\begin{figure*}[h]
\centering
\subfigure[Input]{
	\includegraphics[width=0.25\textwidth]{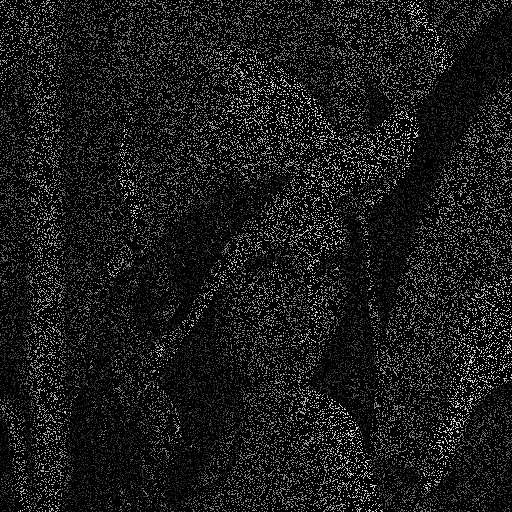}
}
\subfigure[individual, 34.42dB]{
	\includegraphics[width=0.25\textwidth]{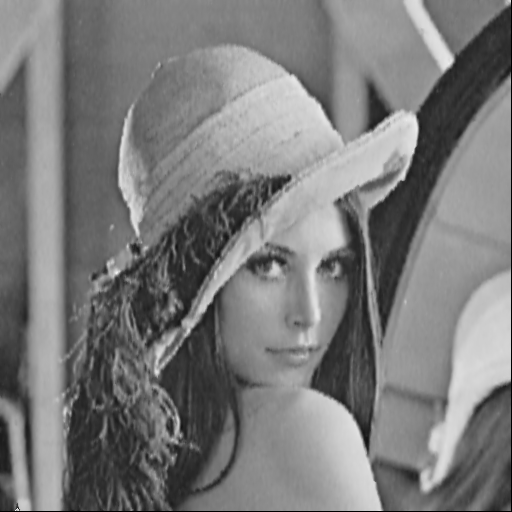}
}
\subfigure[average, 34.62dB]{
	\includegraphics[width=0.25\textwidth]{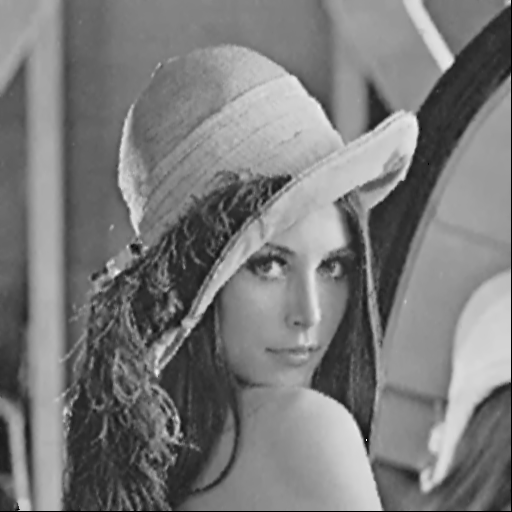}
}

\subfigure[Input]{
	\includegraphics[width=0.25\textwidth]{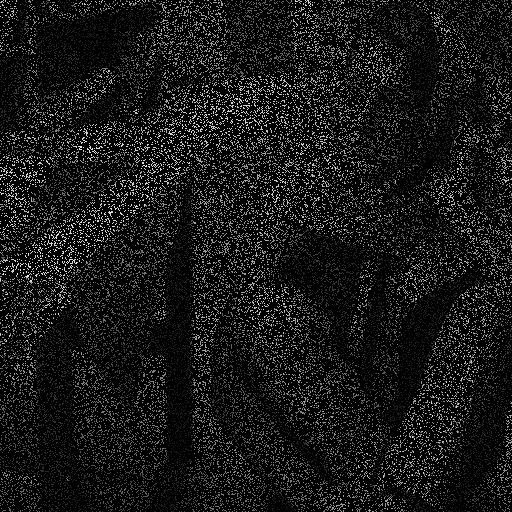}
}
\subfigure[individual, 32.18dB]{
	\includegraphics[width=0.25\textwidth]{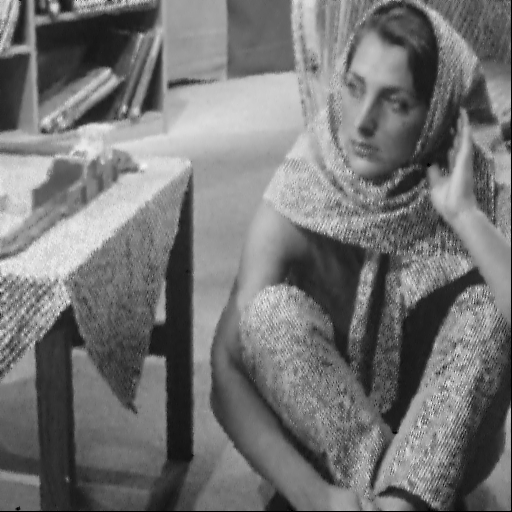}
}
\subfigure[average, 32.25dB]{
	\includegraphics[width=0.25\textwidth]{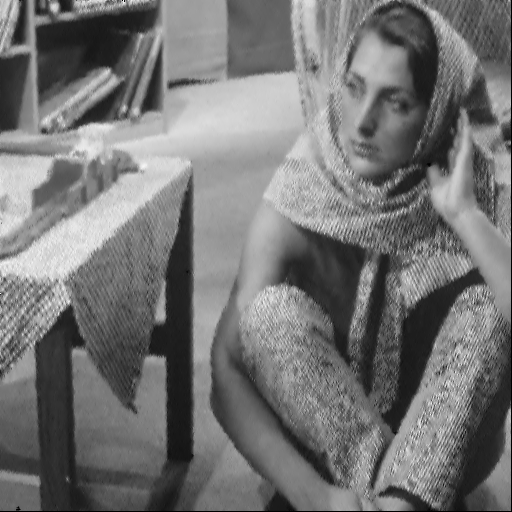}
}

\subfigure[Input]{
	\includegraphics[width=0.25\textwidth]{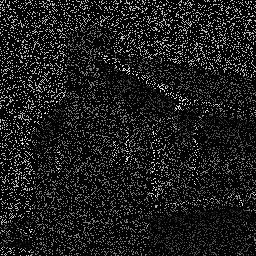}
}
\subfigure[individual, 33.94dB]{
	\includegraphics[width=0.25\textwidth]{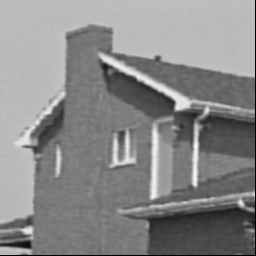}
}
\subfigure[average, 34.14dB]{
	\includegraphics[width=0.25\textwidth]{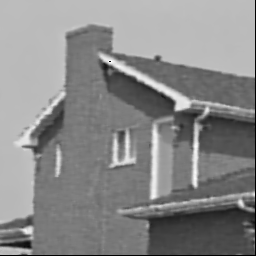}
}

\subfigure[Input]{
	\includegraphics[width=0.25\textwidth]{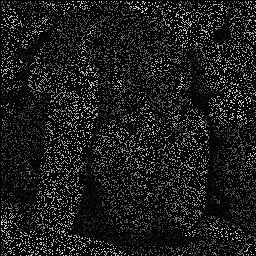}
}
\subfigure[individual, 33.16dB]{
	\includegraphics[width=0.25\textwidth]{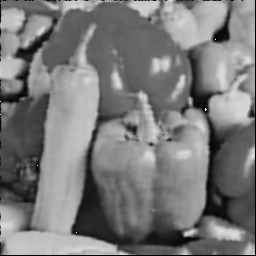}
}
\subfigure[average, 33.40dB]{
	\includegraphics[width=0.25\textwidth]{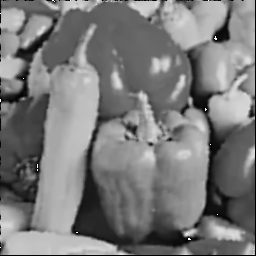}
}
\caption{The results of the image restoration experiment from images with 80\% of pixels missing.}\label{fig:image}
\end{figure*}


\end{document}